\documentclass{article}

\usepackage{float}  
\usepackage{caption}
\usepackage[main,final]{neurips_2026}

\usepackage[utf8]{inputenc} 
\usepackage[T1]{fontenc}    
\usepackage{hyperref}       
\usepackage{url}            
\usepackage{booktabs}       
\usepackage{amsfonts}       
\usepackage{nicefrac}       
\usepackage{microtype}      
\usepackage{xcolor}         
\usepackage{wrapfig}  

\newcommand{\revision}[1]{\textcolor{black}{#1}}

\title{Reinforcement Learning for Diffusion LLMs with Entropy-Guided Step Selection and \\ Stepwise Advantages}

\author{%
  Vishnu Teja Kunde\thanks{Equal contribution.}
  \quad Fatemeh Doudi\footnotemark[1]
  \quad Mahdi Farahbakhsh \\
  \textbf{Dileep Kalathil
  \quad Krishna Narayanan
  \quad Jean-Francois Chamberland} \\[0.5em]
  Department of Electrical and Computer Engineering \\
  Texas A\&M University \\[0.3em]
  {\small\texttt{\{vishnukunde,fatemehdoudi,mahdi.farahbakhsh, dileep.kalathil,krn,chmbrlnd\}@tamu.edu}}}

\usepackage{tikz}
\usepackage[utf8]{inputenc}
\usepackage{amsmath, amssymb, bbm}
\usepackage[]{enumitem}
\usepackage{wrapfig}
\usepackage{subcaption}
\usepackage[dvipsnames]{xcolor}
\usepackage{hyperref}

\usepackage[table]{xcolor}

\newcommand{\abs}[1]{\vert#1\vert}
\newcommand{\norm}[1]{\lVert#1\rVert}

\newcommand{\bE}{\mathbb{E}}

\newcommand{\cV}{\mathcal{V}}

\newcommand{\vm}{\mathbf{m}}

\newcommand{\vq}{\mathbf{q}}

\newcommand{\vx}{\mathbf{x}}

\newcommand{\btheta}{{\boldsymbol{\theta}}}

\newcommand{\pitheta}{{\pi_{\btheta}}}
\newcommand{\pithetaold}{{\pi_{{\btheta}_{\rm old}}}}

\newcommand{\gtheta}{g_{\btheta}}

\usepackage{pgfplots}
\usepackage{xcolor}

\pgfplotsset{compat=1.18}

\definecolor{icmlred}{HTML}{EE4B2B}

\definecolor{icmlblue}{HTML}{031273}
\definecolor{icmlgray}{HTML}{AAAAAA}

\colorlet{icmlblue70}{icmlblue!70!white}  
\colorlet{icmlblue50}{icmlblue!50!white}  
\colorlet{icmlblue30}{icmlblue!30!white}  
\colorlet{icmlblue10}{icmlblue!10!white}

\usepackage{svg}
\usepackage{hyperref}
\usepackage{placeins}

\usepackage[most]{tcolorbox}

\usepackage[capitalize,nameinlink]{cleveref}

\usepackage{algorithm}
\usepackage{algorithmic}

\usepackage{amsmath}
\usepackage{amssymb}
\usepackage{mathtools}
\usepackage{amsthm}

\usepackage{multirow}

\theoremstyle{plain}
\newtheorem{theorem}{Theorem}
\newtheorem{proposition}[theorem]{Proposition}
\newtheorem{lemma}[theorem]{Lemma}

\theoremstyle{definition}

\theoremstyle{remark}
\newtheorem{remark}[theorem]{Remark}

\begin{document}

\maketitle
\begin{abstract}
  \vspace{-0.2cm}
Reinforcement learning (RL) has been effective for post-training autoregressive (AR) language models, but extending these methods to diffusion language models (DLMs) is challenging due to intractable sequence-level likelihoods. Existing approaches therefore rely on surrogate likelihoods or heuristic approximations, which can introduce bias and obscure the sequential structure of denoising. We formulate diffusion-based sequence generation as a finite-horizon Markov decision process over the denoising trajectory and derive the policy gradient that decomposes over denoising steps in terms of stepwise advantages, without requiring explicit evaluation of the sequence likelihood. Grounded in this result, we develop tractable approximations for large-scale training: (i) denoising steps are selected for policy updates via an entropy-guided approximation bound, and (ii) stepwise advantages are estimated using a one-step denoising completion naturally provided by the diffusion model, avoiding costly multi-step rollouts or auxiliary value networks. Experiments demonstrate state-of-the-art results on nearly all benchmarks spanning coding, logical reasoning, and mathematical reasoning, outperforming existing RL post-training approaches for DLMs. Code is available at this \href{https://github.com/vishnutez/egspo-dllm-rl}{repo}.



\end{abstract}

\vspace{-0.3cm}
\section{Introduction}
\label{sec:intro}

\vspace{-0.1cm}
Diffusion Language Models (DLMs) \citep{sahoo2024simple, nie2025llada, ye2025dream} have recently emerged as a compelling alternative to autoregressive language models (ARLMs). Instead of generating tokens sequentially from left to right, DLMs produce text through an iterative denoising process, typically via masked discrete diffusion, enabling bidirectional context and multi-token parallelism for higher token throughput while maintaining competitive output quality. These advantages have spurred rapid progress in DLM architectures and algorithms, including multimodal generation \citep{yang2025mmada, li2025lavida}, long-context modeling \citep{liu2025longllada}, accelerated inference \citep{wu2025fastdllm}, and code generation \citep{song2025seed}. Motivated by the transformative impact of reinforcement learning (RL) post-training on ARLMs \citep{guo2025deepseek}, there is now growing interest in developing effective and scalable RL algorithms tailored to DLMs.

Despite the success of RL post-training for ARLMs, extending these methods to DLMs is \emph{not} a direct translation. RL for ARLMs relies on causal token-wise factorization, which yields a natural token-space Markov decision process (MDP) and enables efficient computation of log-likelihoods and importance ratios. DLMs fundamentally break this structure: generation proceeds through a denoising trajectory in masked space, and the likelihood of the final output does not admit a readily usable token-wise decomposition. As a result, na\"ively porting standard policy-gradient objectives leads to intractable or prohibitively expensive likelihood evaluations. At the same time, diffusion generation offers opportunities largely absent in ARLMs. Model uncertainty evolves non-uniformly along the denoising trajectory, motivating \emph{stepwise credit/advantage assignment} and \emph{stepwise compute allocation} across diffusion steps. Moreover, masked DLMs output token distributions for all positions at each step, enabling a lightweight ``full unmasking'' that can serve as an increasingly accurate proxy of the eventual output later in the trajectory, and thus provide partial learning signals for intermediate steps without an explicit pretrained value function. Together, these challenges and opportunities motivate a principled approach to RL for DLMs that exploits diffusion structure.

\begin{wrapfigure}{r}{0.6\textwidth}
    \vspace{-0.5cm}
    \centering
    \begin{subfigure}[t]{0.49\linewidth}
        \centering
        \includegraphics[width=\linewidth]{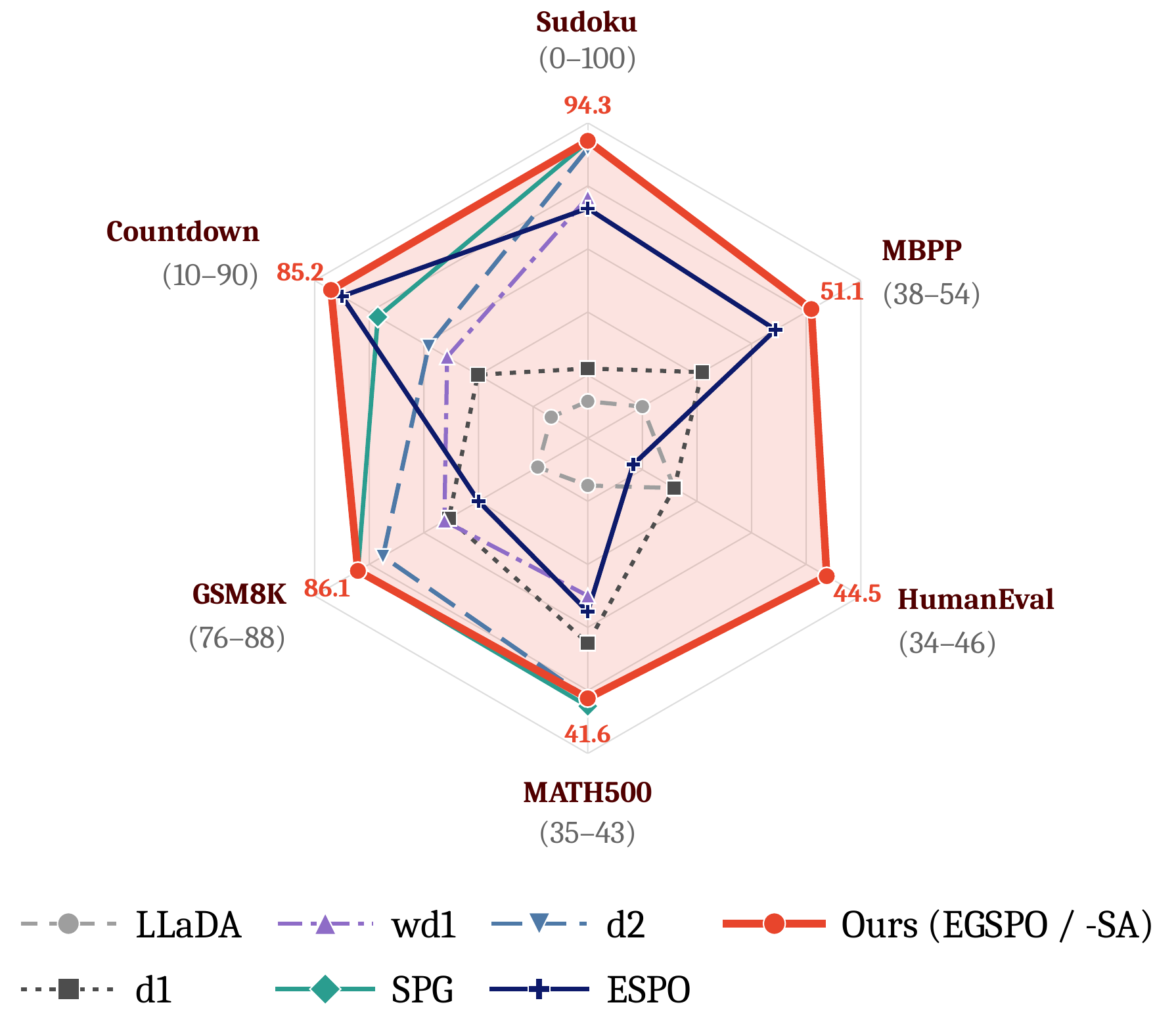}
        \caption{}
        \label{fig:intro_llada}
    \end{subfigure}
    \hfill
    \begin{subfigure}[t]{0.49\linewidth}
        \centering
        \includegraphics[width=\linewidth]{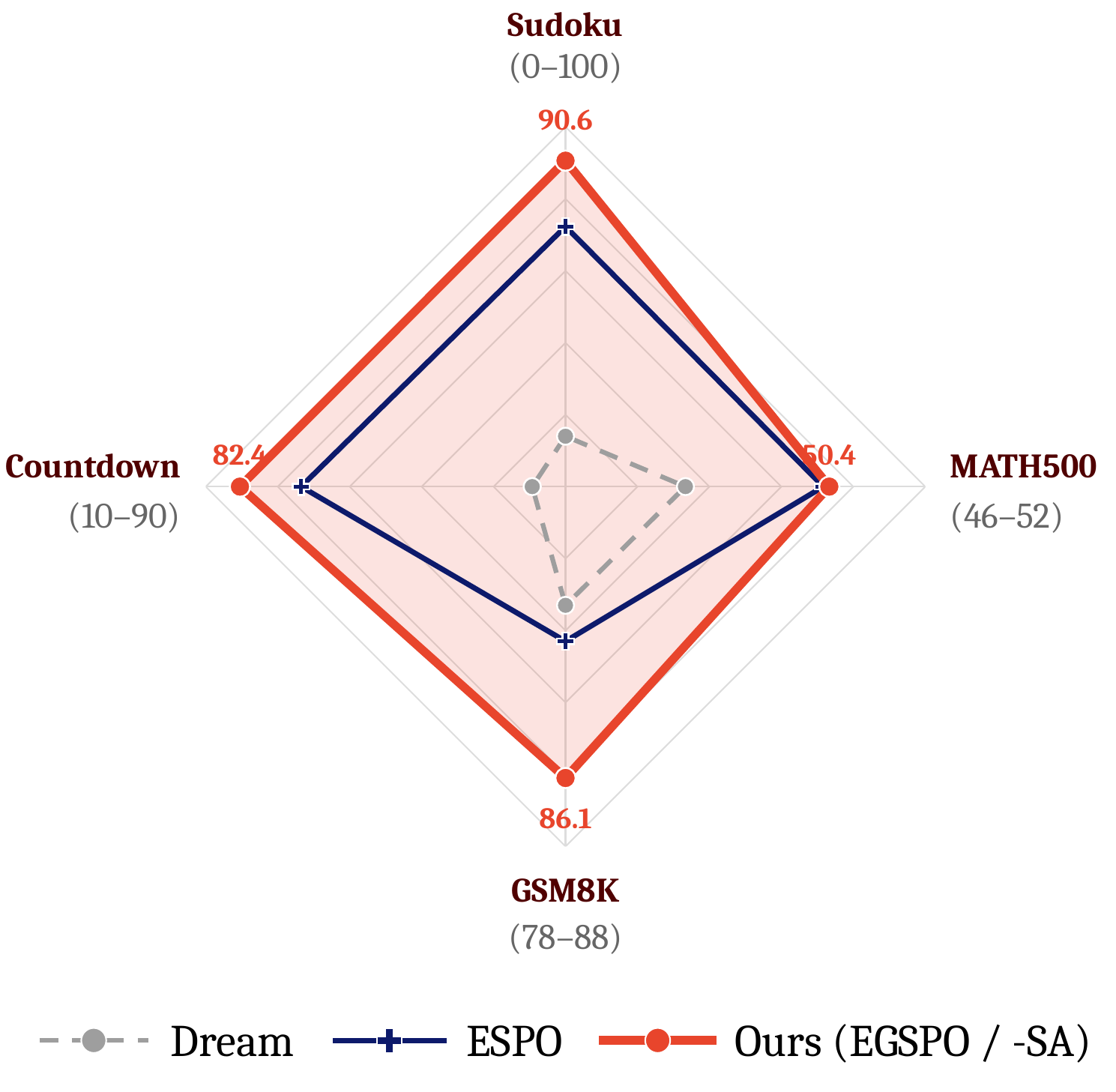}
        \caption{}
        \label{fig:intro_dream}
    \end{subfigure}
    \caption{Benchmark accuracy for (a) LLaDA-8B-Instruct and (b) Dream-7B-Instruct. Ours reports the better of EGSPO and EGSPO-SA per task, using the best accuracy across generation lengths. Axis ranges are shown beside each benchmark. See Tables~\ref{tab:main_results}, \ref{tab:dream_results}, and~\ref{tab:coding_results} for results and evaluation protocols.}
    \label{fig:intro_results}
    \vspace{-0.4cm}
\end{wrapfigure}


A growing body of work has begun exploring RL post-training for DLMs via surrogate objectives and tractable likelihood approximations \citep{zhao2025d1, tang2025wd1, wang2025spg, ou2026principled, wang2025d2, wang2025revolutionizing, zhao2025diffpotrainingdiffusionllms}. While practically effective, these methods typically begin with a \emph{chosen approximation} to the likelihood, objective, or advantage and optimize the resulting surrogate directly, leaving the connection to the true RL objective less explicit and step-level optimization across the denoising trajectory largely implicit (see Section~\ref{sec:related-work} for detailed comparison).

We take a complementary, first-principles approach that makes the diffusion structure explicit, rather than treating a DLM as a black-box sampler. Instead of beginning with a particular surrogate likelihood approximation, we ask the fundamental questions that should underlie any principled RL method: What is the right MDP formalism for DLMs? What is the policy gradient for the true RL objective? How can it be approximated tractably at scale? And how can diffusion-time structure enable \emph{stepwise} advantage (credit) estimation and \emph{compute allocation} across denoising steps? We answer these questions in the affirmative, leading to the following main contributions.

\vspace{-0.2cm}
\begin{itemize}[leftmargin=*, itemsep=0.35em]


\item \revision{\textbf{Policy gradient with stepwise advantages:}
Using a finite-horizon MDP formulation of masked diffusion generation, we derive a policy gradient that decomposes over denoising steps, yielding a principled notion of \emph{stepwise advantages} without requiring explicit evaluation of the final sequence likelihood.}

\item \textbf{Tractable estimators exploiting diffusion structure:} We turn this theory into a practical algorithm by exploiting two DLM-native capabilities absent in ARLMs. First, we allocate training compute \emph{across denoising steps} using the model’s intrinsic uncertainty: higher-entropy steps are prioritized under a fixed budget. We call this method \textbf{E}ntropy-\textbf{G}uided \textbf{S}tepwise \textbf{P}olicy \textbf{O}ptimization (\textbf{EGSPO}). Second,  we estimate stepwise advantages via a lightweight full-sequence ``one-shot'' completion from intermediate states, yielding intermediate learning signals without an extra value network or costly multi-step rollouts. We call this method  \textbf{E}ntropy-\textbf{G}uided \textbf{S}tepwise \textbf{P}olicy \textbf{O}ptimization with \textbf{S}tepwise \textbf{A}dvantages (\textbf{EGSPO-SA}). Together, these ideas make RL  training scalable for DLMs while preserving the diffusion-time structure. 

\item \textbf{State-of-the-art results:}
We achieve state-of-the-art results on nearly all benchmarks spanning \emph{coding}, \emph{logical reasoning}, and \emph{mathematical reasoning}, outperforming existing RL post-training approaches for DLMs on  LLaDA \citep{nie2025llada}  and Dream \citep{ye2025dream} models.

\end{itemize}

\vspace{-0.5cm}
\section{Related Work}
\label{sec:related-work}
\vspace{-0.2cm}

\textbf{Diffusion Language Models:} DLMs \citep{sahoo2024simple, nie2025llada, ye2025dream} extend discrete diffusion frameworks \citep{austin2021structured, lou2023discrete, song2025seed} to natural language, enabling parallel multi-token generation. Prior work has explored architectural simplifications \citep{shi2024simplified}, theoretical connections between discrete and continuous diffusion \citep{zhao2024improving}, and multimodal extensions \citep{yang2025mmada, li2025lavida}. Inference efficiency has been improved via remasking schemes \citep{wang2025remaskingdiscretediffusionmodels}, token reordering \citep{arriola2025block, sahoo2025esoteric}, and entropy- or confidence-based sampling \citep{benhamu2025acceleratedsamplingmaskeddiffusion, wu2025fastdllm}.


\textbf{RL for Diffusion Image Models:} RL is used for preference alignment in image diffusion models via policy-gradient methods \citep{black2024training, fan2023dpok, uehara2024fine} or preference-based optimization \citep{wallace2024diffusion, yang2024cvpr}. DDPO~\citep{black2024training} similarly casts denoising as an MDP, but derives the policy gradient only in reward form without stepwise value functions or advantages. Furthermore, image diffusion models are orders of magnitude smaller than 7--8B DLMs, so running all $T$ denoising steps per update is affordable there; our entropy-guided step selection addresses the budget constraint that arises specifically at LLM scale. 

\textbf{RL for Diffusion Language Models:} RL post-training for DLMs has grown rapidly. \emph{d1}~\citep{zhao2025d1} adapts GRPO via a mean-field likelihood approximation; \emph{wd1}~\citep{tang2025wd1} stabilizes training with weighted policy updates while retaining similar likelihood proxies; \emph{SPG}~\citep{wang2025spg} optimizes pessimistic/optimistic surrogate objectives derived from likelihood bounds; \emph{ESPO}~\citep{ou2026principled} treats entire sequence generation as a single action and uses the ELBO as a tractable sequence-level likelihood proxy with per-token importance ratio normalization; \emph{TraceRL}~\citep{wang2025revolutionizing} trains a trajectory-aware diffusion value model; and \emph{DiFFPO}~\citep{zhao2025diffpotrainingdiffusionllms} uses surrogate policies with importance sampling corrections. Despite their empirical success, these methods begin with a chosen approximation to the likelihood, objective, or advantage and optimize it directly, leaving step-level credit assignment largely implicit and the connection to the true RL objective opaque. The most closely related work is \emph{d2}~\citep{wang2025d2}, which derives a reward-weighted policy gradient without an explicit MDP, using step-merging estimators that collapse multiple denoising steps into a single update and discard per-step structure; our entropy-guided step selection instead retains the most informative individual steps via Lemma~\ref{lem: error_bound_using_entropy}. \emph{TraceRL}~\citep{wang2025revolutionizing} is conceptually similar to \emph{d2} but evaluates on non-standard model architectures, so we treat \emph{d2} as the canonical trajectory-level baseline. \citet{hong2026improving} train an unmasking-head policy, focusing on the unmasking mechanism rather than finetuning of the model for a specific task. \revision{\citet{xie2025denoising} augment Diffu-GRPO with process rewards that measure the contribution of denoising intervals to final correctness. Our method instead estimates stepwise advantages by comparing the final reward with rewards of greedy one-step completions from intermediate states.}

\vspace{-0.4cm}
\section{Preliminaries and Problem Formulation}
\vspace{-0.3cm}

\subsection{Preliminaries}
\label{sec:prelims}
\vspace{-0.1cm}

\textbf{Notations.} Let $\vx = \vx_{0:T,1:L}$ denote a sequence of $T+1$ steps of $L$-length sentences. Each of $\vx_t$ can take values from a finite set of vocabulary $\cV \cup \{\vm\}$ where $\vm$ is a special token denoting mask. The state $\vx_0$ denotes the clean sentence and can take values only in $\cV$.


\textbf{Masked Diffusion Language Model (MDLM):} MDLMs \citep{sahoo2024simple, nie2025llada, song2025seed, ye2025dream} generate text via iterative unmasking. A clean sequence $\vx_0$ is masked as $\vx_{t,\ell} \sim {\rm Cat}(\cdot \mid \alpha_t \vx_{0,\ell} + (1-\alpha_t)\vm)$ with decreasing $\alpha_t$. The reverse process is parameterized by $f_\btheta$: given $\vx_{t+1}$, denoised samples follow $\vx_{s,\ell} \sim \pi_\btheta^{s\mid t}(\cdot \mid \vx_t) = q^{s\mid t,0}(\cdot \mid \vx_t, \vx_0 = f_\btheta(\cdot \mid \vx_t))$, and iterative generation factorizes as $\pitheta^{s:t-1\mid t}(\vx_{s:t-1}\mid\vx_t) \triangleq \prod_{u=t-1}^{s}\pitheta^{u\mid u+1}(\vx_u\mid\vx_{u+1})$. Let $M_t(\vx_{t+1}) \triangleq \{\ell: \vx_{t+1,\ell}=\vm\}$ denote masked positions at step $t$. Since $f_\btheta$ models a product distribution over masked positions, it induces a one-step distribution $\pitheta^{0\mid t}(\vx_0\mid\vx_t) = \prod_{\ell:\vx_{t,\ell}=\vm} f_{\btheta,\ell}(\vx_{0,\ell}\mid\vx_t)$. We drop time-step superscripts on $\pitheta$ where clear from context.
\revision{The denoising network $f_\btheta$ is trained to predict clean tokens $\vx_0$ from a partially masked sequence, rather than directly predicting the next intermediate state. A masked-token reconstruction objective takes the form}
{\color{black}
\begin{align}
\mathcal{L}_{\rm denoise}(\btheta)
= -\bE_{\vx_0,u,\vx_u}
\left[
w_u \sum_{\ell:\vx_{u,\ell}=\vm}
\log f_{\btheta,\ell}(\vx_{0,\ell}\mid\vx_u)
\right].
\label{eq:denoising_training}
\end{align}
}
\revision{For LLaDA~\citep{nie2025llada}, $u\sim\mathrm{Unif}(0,1)$ is the masking probability, $w_u=1/u$, and $\vx_0$ is drawn from the training data. During generation, the sampler uses these clean-token predictions to construct the next partially unmasked state.}





\textbf{RL for LLMs:} Given a query $\vq\sim\mathcal{D}$, an LM generates an output $\vx_0$ and receives a reward $r(\vx_0,\vq)$. RL fine-tuning seeks to maximize the expected reward: 
\begin{align}
    J(\btheta) &\triangleq \bE_{\vq,\vx_0 \sim \pitheta(\cdot \mid \vq)}[r(\vx_0,\vq)]. \label{eq:ar_rl_objective}
\end{align}
Policy-gradient methods are standard tools for optimizing \eqref{eq:ar_rl_objective}. Among them, the Group Relative Policy Optimization (GRPO) \citep{shao2024deepseekmath} is widely popular for RL with Verifiable Reward (RLVR) tasks, as it eliminates the need for a trained value model used in  PPO-style algorithms \citep{schulman2017ppo}. The GRPO loss function for an  ARLM is given by
\begin{align*}
    L_{\rm AR}(\btheta) = -\bE_{\vq,\vx_0\sim\pithetaold(\cdot \mid \vq)}\!\left[\frac{1}{L} \sum_{\ell=1}^L \min\!\left(\rho_\ell A^{\pithetaold},\, {\rm clip}_{1-\epsilon}^{1+\epsilon}(\rho_\ell)A^{\pithetaold}\right) \right] + \beta D_{\rm KL}(\pitheta \Vert \pi_{\rm ref}),
\end{align*}
where $\rho_\ell = \frac{\pitheta(\vx_{0,\ell} \mid \vx_{0,<\ell},\vq)}{\pithetaold(\vx_{0,\ell} \mid \vx_{0,<\ell},\vq)}$ is the importance ratio, and the group-relative advantage $A^{\pithetaold}(\vx_0^j, \vq) = r(\vx_0^j, \vq) - \frac{1}{G} \sum_{i=1}^G r(\vx_0^i, \vq)$ is estimated from $G$ completions $\vx_0^j \sim \pithetaold(\cdot \mid \vq)$.

\vspace{-0.1em}
\subsection{Problem Formulation: RL for Diffusion LLMs}
\vspace{-0.1cm}

RL finetuning of an AR-LLM exploits its auto-regressive structure by modeling the token generation as a token-space MDP, and the LLM as a policy. Since $\vx_0 = \vx_{0,1:L}$, we can define the  initial state as $\vq$, state at step $\ell$ as $(\vq, \vx_{0,<\ell})$ and action at step $\ell$ as the next token $\vx_{0,\ell} \sim \pi_{\btheta}(\cdot \mid \vx_{0,<\ell}, \vq)$. Any policy gradient algorithm for solving (\ref{eq:ar_rl_objective}) has to calculate the term $\pitheta(\vx_0 \mid \vq)$, and here it decomposes as $\pitheta(\vx_{0} \mid \vq) = \prod_{\ell=1}^L \pitheta(\vx_{0,\ell} \mid \vx_{0,<\ell}, \vq).$ Due to causal attention of AR-LLM, all these terms can be efficiently computed using only a single model forward pass. \textbf{Unfortunately,  DLMs break these conveniences} because there is no token-wise causal factorization of $\pitheta(\vx_0 \mid \vq)$ available. Naively porting the RL training methods for AR-LLM would require evaluating (explicitly or implicitly) sequence-level likelihood terms that are intractable as there is no direct decomposition over tokens. Prior work addresses this with surrogate objectives and likelihood approximations \citep{zhao2025d1,tang2025wd1,wang2025spg,wang2025d2,ou2026principled}, but these approaches necessarily make compromises: the optimized objective may deviate from \cref{eq:ar_rl_objective}, gradients may be biased, and the role of \emph{individual denoising steps} in producing the final reward is often ignored.

In this work, we take a complementary approach. Rather than beginning with a particular surrogate likelihood approximation, we ask the foundational questions that should underlie any principled RL method for diffusion LMs. 
$(i)$ \textit{MDP formalism:} What are the states, actions, and transitions that faithfully represent diffusion-based generation while remaining amenable to RL analysis? 
$(ii)$ \textit{Policy-gradient for \cref{eq:ar_rl_objective}:}  Can we derive an \emph{unbiased} gradient expression that respects the sequential structure of denoising? How can we estimate that gradient tractably?
$(iii)$ \textit{Stepwise advantages:} In AR decoding, actions are tied to token positions; in diffusion decoding, actions are tied to {denoising step}. Can diffusion structure yield improved stepwise advantages that do not have direct AR analogues? We provide positive answers to these questions in  \cref{sec:methodology}.

\vspace{-0.2cm}
\section{Methodology}
\label{sec:methodology}

\subsection{Diffusion MDP and Policy Gradient}
\label{sec:pg-theorem}

We formalize the unmasking process as a $T$-step MDP indexed by diffusion time $t$ running from $T$ down to $0$. The state at diffusion step $t$ is $(\vx_t, \vq)$, the action is the next denoised state $\vx_{t-1}$, and transitions are deterministic given the policy. The reward is zero at all intermediate steps and equals $r(\vx_0, \vq)$ at the terminal step $t = 0$. The value of a policy $\pi$ at time step $t$ is given by
\begin{align}
\label{eq:value-defn}
    V_t^{\pi}(\vx_t, \vq) = \bE_{\vx_{< t} \sim \pi^{< t \mid t}(\cdot \mid \vx_t, \vq)}[r(\vx_0, \vq)],
\end{align}
where $\vx_{< t} \sim \pi^{< t \mid t}(\cdot \mid \vx_t, \vq)$ denotes that the expectation is with respect to the trajectories that start at $\vx_t$ and generate $\vx_0$ in $t$ steps. The value of a policy $V^\pi$, averaged over the initial state, is then defined as $V^\pi = \bE_{\vq}[ V_T^{\pi}(\vm, \vq)]$. By comparing \cref{eq:ar_rl_objective} and \cref{eq:value-defn}, it is straightforward to see that $J(\btheta) = V^{\pitheta}$, showing that our MDP formalism will indeed solve the finetuning problem.  We now derive the policy gradient for the diffusion MDP.

\begin{proposition}[Policy Gradient]
\label{thm:pg_thm}
    The gradient of the objective $J(\btheta)$ (c.f. \cref{eq:ar_rl_objective}) is given by
    \begin{align}
        &\nabla_{\btheta} J(\btheta) = \bE_{\vq,\vx \sim \pitheta(\cdot \mid \vq)}\!\left[r(\vx_0, \vq) \nabla_{\btheta} \log \pitheta(\vx \mid \vq)\right] \label{eq:pg_reward} \\
        &=~\sum_{t=0}^{T-1}\bE_{\vq,\vx\sim \pitheta}[{A}^{\pitheta}_t(\vx_{t+1},\vx_0,\vq) \nabla_{\btheta} \log \pitheta(\vx_t \mid \vx_{t+1})], \label{eq:pg_adv}
    \end{align}
    where the stepwise advantage is given by
    \begin{align}
        A^{\pitheta}_t(\vx_{t+1}, \vx_0, \vq) \triangleq r(\vx_0, \vq) - V_{t+1}^\pitheta(\vx_{t+1}, \vq).
    \end{align}
\end{proposition}

Here $\vx \triangleq \vx_{0:T}$ denotes the full denoising trajectory and $\log \pitheta(\vx \mid \vq) = \sum_{t=0}^{T-1} \log \pitheta(\vx_t \mid \vx_{t+1})$. The proof is given in Appendix~\ref{proof:pg_thm}.

\begin{remark}[Variance reduction via per-step baseline]
Both forms are exact. \cref{eq:pg_reward} broadcasts raw reward uniformly; \cref{eq:pg_adv} subtracts the per-step baseline $V_{t+1}^\pitheta$, concentrating the signal on steps where the policy's expected return deviates from the final outcome and reducing variance. Estimating $V_{t+1}^\pitheta$ tractably is addressed in \cref{sec:EstimatingAdvantages}.
\end{remark}

\begin{remark}
\label{rmk:position_selection}
Let $U_t \triangleq \{\ell: \vx_{t,\ell} \neq \vx_{t+1,\ell}\}$ denote the positions unmasked at step $t$. Since the unmasked tokens are conditionally independent given $\vx_{t+1}$, the log-policy decomposes as $\log \pitheta(\vx_t \mid \vx_{t+1}) = \log \pitheta(U_t \mid \vx_{t+1}) + \sum_{\ell \in U_t} \log f_{\btheta,\ell}(\vx_{t,\ell} \mid \vx_{t+1})$. In practice, $U_t$ is chosen by confidence-based decoding: unmasking the positions with the highest $\max_v f_{\btheta,\ell}(v \mid \vx_{t+1})$. The position-selection term $\nabla_\btheta \log \pitheta(U_t \mid \vx_{t+1})$ is not differentiable in the standard sense, since $U_t$ is determined by an argmax (top-$k$) over confidence scores. We therefore treat $U_t$ as fixed when differentiating, following \citet{wang2025d2}, yielding $\nabla_{\btheta} \log \pitheta(\vx_t \mid \vx_{t+1}) = \sum_{\ell \in U_t} \nabla_{\btheta} \log f_{\btheta,\ell}(\vx_{t,\ell} \mid \vx_{t+1})$ and $\pi(\vx_t \mid \vx_{t+1}) = \prod_{\ell \in U_t} f_{\btheta, \ell}(\vx_{t,\ell} \mid \vx_{t+1})$. This is an empirical design choice. However, the resulting gradient still aligns with the correct objective in the following sense: improving the token log-likelihood at a selected position $\ell$ increases $\max_v f_{\btheta,\ell}(v \mid \vx_{t+1})$, which reinforces that position being selected by confidence-based decoding in future rollouts, so the two components of the gradient are implicitly coupled through training.
\end{remark}

\begin{figure*}[h]
    \centering
    \begin{subfigure}{0.56\linewidth}
        \centering
        \includegraphics[width=\linewidth]{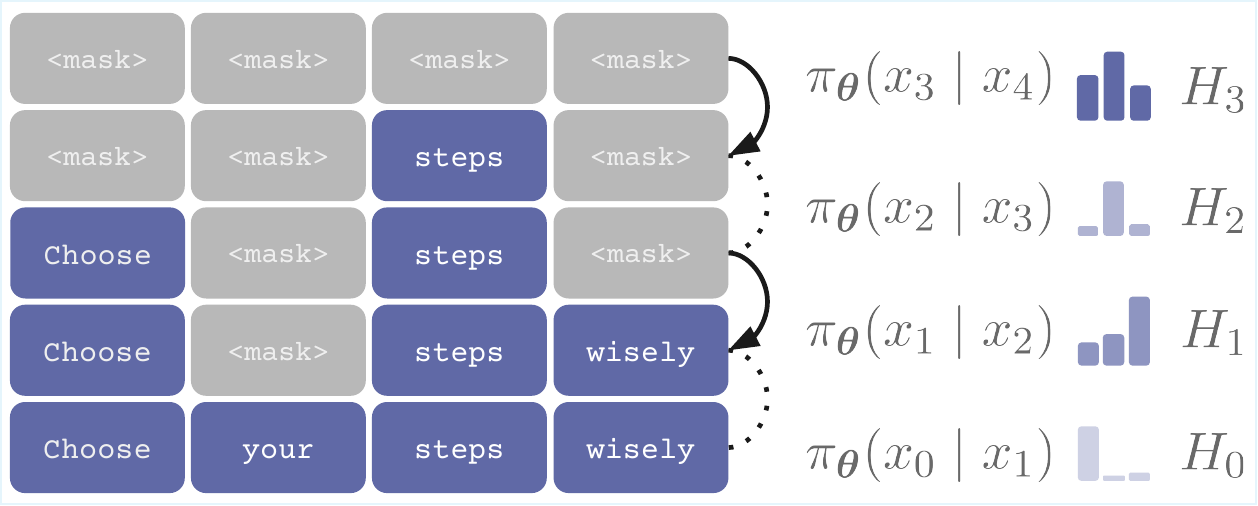}
        \caption{\textbf{Entropy-Guided Step Selection}}
    \end{subfigure}
    \hfill
    \begin{subfigure}{0.42\linewidth}
        \centering
        \includegraphics[width=\linewidth]{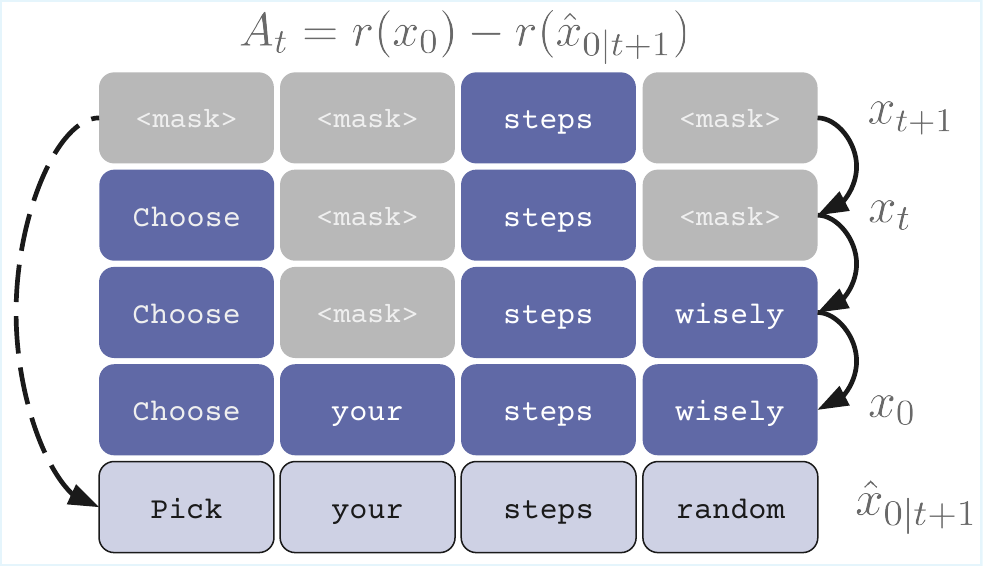}
        \caption{\textbf{Stepwise Advantage Estimation}}
    \end{subfigure} 
    \caption{ (a) \textbf{Entropy-guided step selection}: $K$ steps with highest entropy $H_t$ are selected (solid lines); with $H_3>H_1>H_2>H_0$ and $K{=}2$, steps $3$ and $1$ are chosen.
    (b) \textbf{Stepwise advantage estimation}: Reward of greedy one-step completion $\hat{\vx}_{0\mid t+1}$ from $\vx_{t+1}$ approximates $V_{t}$; $A_t$ measures the additional reward from continuing the full denoising trajectory to $\vx_0$.}
\vspace{-0.4cm}
\end{figure*}

\subsection{Entropy-Guided Step Selection}
\label{sec:entropy_step_selection}


In ARLMs, given a sequence $\vx_0$, the network output at position $\ell$ models the likelihood $\pitheta(\vx_{0, \ell} \mid \vx_{0, <\ell})$ for all $\ell \in [1:L]$ due to the causal attention mask. Thus, a single forward pass gives all the required likelihoods.  This convenience is broken in diffusion LM due to its bidirectional attention, and it can compute only the term  $\pitheta(\vx_t \mid \vx_{t+1})$ in the forward pass of $f_\btheta(\cdot \mid \vx_{t+1})$. The policy gradient in \cref{eq:pg_adv} has $T$ such terms, each requiring a separate forward pass through the network. This is computationally expensive as $T$ is typically of the order of $10^2-10^3$. 

To overcome this challenge, we take the gradient only at a subset $S \subseteq [0:T-1]$ with $|S| \le K$. Writing $\nabla_\btheta J_t(\btheta) = \bE_{\vq,\vx\sim \pitheta}[{A}^{\pitheta}_t \nabla_{\btheta} \log \pitheta(\vx_t \mid \vx_{t+1})]$, so that $\nabla_\btheta J(\btheta) = \sum_{t} \nabla_\btheta J_t(\btheta)$, the subset approximation is $\nabla_\btheta J_S(\btheta) = \sum_{t\in S} \nabla_\btheta J_t(\btheta)$. Two intuitive heuristics are: \textit{(i) Random}: sample $K$ steps uniformly at random; \textit{(ii) Uniform}: take $S = \{kT/K : k \in [0{:}K{-}1]\}$. However, these approaches overlook the structure of the policy at each denoising step along the trajectory. A more principled approach would be to choose $S$ with $\abs{S} \le K$ such that the gradient approximation error $\Delta_S = \norm{\nabla_\btheta J(\btheta)-\nabla_\btheta J_S(\btheta)}$ is minimized. Since evaluating this objective $\Delta_S$ directly is expensive, we instead minimize a surrogate objective. 

\begin{lemma}
\label{lem: error_bound_using_entropy}
    Let $\pitheta^{t\mid t+1}$ be the softmax policy with logits $g_\btheta^{t,i}$ over $i\in [\abs{\cV}]$ for each $t\in [0:T-1]$. Let $\norm{\nabla_\btheta g_\btheta^t} \triangleq \max_{i,j} \norm{\nabla_\btheta (g_\btheta^{t,i} - g_\btheta^{t,j})}$ and assume that $\max_t \norm{\nabla_\btheta g_\btheta^t} \le B$. Let $H(\pitheta^{t\mid t+1})$ be the entropy of $\pitheta(\cdot \mid \vx_{t+1})$.  Defining $\Delta_S = \norm{\nabla_\btheta J(\btheta)-\nabla_\btheta J_S(\btheta)}$, we have
\begin{align}
    \Delta_S \le B (\max_{t\in [T]} \abs{A_t^{\pi_\btheta}} ) \sum_{t \notin S} H(\pitheta^{t\mid t+1}). \label{ineq:delta_s_ub}
\end{align}
\end{lemma}
Proof is given in \cref{app:theoretical_results}.

\begin{remark}
The factor $\max_t |A_t^{\pi_\btheta}|$ in \cref{ineq:delta_s_ub} is bounded by $r_{\max} - r_{\min}$, the range of the reward function. In practice, rewards are bounded in a known interval $[-M, M]$.
\end{remark}



Using the upper-bound in \cref{ineq:delta_s_ub} as the surrogate objective for $\Delta_S$,  we solve the following optimization problem to select time steps $S$ such that 
\begin{align}
    \label{eq:step-selection-opt}
    \min_{S \subseteq [0:T-1]} \sum_{t \notin S} H(\pitheta^{t\mid t+1}), \quad \abs{S} \le K.
\end{align}
The solution of \cref{eq:step-selection-opt} is given by the greedy choice:
\begin{align}
    \label{eq:step-selection-soln}
    S^* = {\rm top}[K](t: H(\pitheta^{t\mid t+1}))
\end{align}

\revision{Intuitively, at these steps the policy has the highest uncertainty, so allocating gradient computation to them minimizes the upper bound on the gradient approximation error $\Delta_S$. This bound may be loose and does not guarantee that the selected subset minimizes $\Delta_S$ itself. A diagnostic comparison with full-step updates in Appendix~\ref{app:gradient_diagnostic} further supports the empirical effectiveness of entropy-guided selection.}

\subsection{Estimating Stepwise Advantages}
\label{sec:EstimatingAdvantages}

To compute the stepwise advantage $A_t^\pitheta$, we need to evaluate the value function $V_t^{\pitheta}$ for each $t$ (c.f. \cref{eq:value-defn}). Evaluating it using Monte Carlo sampling will require generating multiple trajectories that start at $\vx_t$ and reach $\vx_0$ in $t$ steps, which requires a large number of additional forward passes through the model. To overcome this difficulty, we leverage a structure unique to DLMs: $f_\btheta$ directly models a distribution over clean tokens $\vx_0$ from any masked state $\vx_t$ via the one-step denoising distribution $\pitheta^{0\mid t}$ (c.f. \cref{sec:prelims}). \textit{This is a capability ARLMs do not possess without a separate value network.} We use it to propose:
\begin{align}
    \label{eq:value_approx}
    \hat{V}^{\pitheta}_t(\vx_t, \vq) &= \bE_{\vx_0 \sim \pitheta^{0 \mid t}(\cdot \mid \vx_t, \vq)}[r(\vx_0, \vq)].
\end{align}
This expectation is further approximated by the reward evaluated for a single greedy completion $\hat{\vx}_{0\mid t}$, \textbf{available at no extra cost} from the forward pass already computed during rollout:
\begin{align}
    \hat{\vx}_{0, \ell}(\vx_{t}) =
    \begin{cases}
    \arg\max_{v} f_{{\btheta}, \ell}(v \mid \vx_{t}),&\vx_{t, \ell} = \vm \\
    \vx_{t, \ell},&{\rm otherwise.}
    \end{cases}
\label{eq:greedy_completion}
\end{align}

The approximation $\hat{V}_{t+1}^\pitheta$ is biased: the one-step surrogate $\pitheta^{0\mid t+1}$ diverges from the full $t$-step rollout when many tokens remain masked (large $t$), with the bias diminishing as $t \to 0$. \revision{Using $\hat{V}_{t+1}^\pitheta$, we define the approximate stepwise advantage:}
\begin{align}
    \hat{A}_t^\pitheta(\vx_{t+1}, \vx_0, \vq) \triangleq r(\vx_0, \vq) - \hat{V}_{t+1}^\pitheta(\vx_{t+1}, \vq) \approx r(\vx_0, \vq) - r(\hat{\vx}_{0\mid t+1}, \vq). \label{eq:approx_adv}
\end{align}

\revision{This estimate uses the reward of a greedy one-step completion as a state-dependent baseline. Although this baseline generally differs from the true value function, that discrepancy alone does not bias the full on-policy score-function gradient: any action-independent baseline, treated as fixed when differentiating, has zero expected contribution. Its quality instead affects the variance of the gradient estimate, and variance reduction is not guaranteed. We combine the terminal reward and the estimated stepwise advantage through the effective per-step coefficient}
\begin{align}
    r(\vx_0,\vq) + \lambda\,\hat{A}_t^\pitheta(\vx_{t+1},\vx_0,\vq),
    \label{eq:effective_coeff}
\end{align}
\revision{where $\lambda \ge 0$ controls the contribution of the stepwise term. Writing this coefficient as $(1+\lambda)r-\lambda\hat{V}_{t+1}^\pitheta$, the full on-policy update, before step selection and surrogate optimization, equals $(1+\lambda)\nabla_\btheta J(\btheta)$ in expectation under the conditions of Proposition~\ref{thm:pg_thm}. Thus, up to an overall scale, $\lambda$ controls the baseline weight. This identity does not establish unbiasedness of the practical update, which additionally uses entropy-guided step selection, group centering, clipping, and KL regularization.}

\textit{Intuitively}, $\hat{A}_t^\pitheta$ measures how much better the full trajectory rollout from $\vx_{t+1}$ performed relative to a one-step greedy completion: positive means iterative denoising from $\vx_{t+1}$ was crucial; near-zero means the one-step prediction already captures the outcome. This per-step signal is obtained at no extra cost from the rollout forward pass, with no auxiliary value network required. \textit{To the best of our knowledge, we are the first to derive stepwise advantages from this DLM-specific structure in a principled way for RL fine-tuning.}
\revision{\paragraph{When is the stepwise signal useful?}
We expect the stepwise term to help when the reward of a greedy one-step completion from an intermediate state reliably predicts the terminal reward. In tasks such as Sudoku and Countdown, intermediate states progressively fix parts of the solution, making this signal useful. In mathematical reasoning, the final answer may depend on reasoning tokens that remain unresolved, making intermediate greedy completions less reliable. This provides an intuition for the consistent gains of EGSPO-SA on logical reasoning and its mixed results relative to EGSPO on mathematical reasoning.}

\subsection{From Policy Gradient to GRPO Loss}

For $G$ sampled completions $\{\vx^j\} \sim \pithetaold(\cdot \mid \vq)$ and each selected step $t \in S$, the effective per-step coefficient from \cref{eq:effective_coeff} is instantiated via \cref{eq:greedy_completion,eq:approx_adv} as $\tilde{A}_t^j = r(\vx_0^j,\vq) + \lambda\,\hat{A}_t^{\pithetaold}(\vx^j,\vq)$. We apply GRPO-style centering across completions for each selected step, $\hat{\bar{A}}_t^j = \tilde{A}_t^j - \tfrac{1}{G}\sum_{i=1}^G \tilde{A}_t^i$. The training loss is a sum over selected steps, with each term applying PPO-style clipping to its own importance ratio $\rho_t = \pitheta(\vx_t \mid \vx_{t+1})/\pithetaold(\vx_t \mid \vx_{t+1})$ plus a KL penalty, consistent with the sum-over-$t$ structure of the policy gradient in \cref{eq:pg_adv}.

We refer to this method as \emph{Entropy-Guided Stepwise Policy Optimization with Stepwise Advantages} (\textbf{EGSPO-SA}). Setting $\lambda = 0$ recovers entropy-guided step selection with the sequence-level reward alone, denoted \textbf{EGSPO}.

\begin{remark}[$\lambda$ across training]
In all experiments, $\lambda$ is fixed within each diffusion episode and does not depend on the denoising step $t$. Separately, $\lambda$ may be annealed across training steps for some tasks, analogous to standard hyperparameter scheduling; per-task values are in Appendix~\ref{sec:hypsetting}.
\end{remark}
\paragraph{Exact gradient and practical approximations.}
\revision{Proposition~\ref{thm:pg_thm} establishes the exact policy gradient for the diffusion MDP. To obtain a tractable algorithm, we introduce three approximations: (i) we treat the confidence-based position-selection set $U_t$ as fixed when differentiating, as discussed in Remark~\ref{rmk:position_selection}; (ii) we replace the full sum over denoising steps with a subset of $K$ steps selected using the entropy-based upper bound in Lemma~\ref{lem: error_bound_using_entropy}; and (iii) we estimate the stepwise value functions using rewards of greedy one-step completions. These estimates are then incorporated into a GRPO-style surrogate with group centering, PPO-style clipping, and KL regularization. Thus, Proposition~\ref{thm:pg_thm} provides the theoretical foundation, while the practical training update incorporates the above approximations and surrogate optimization choices.}
\section{Experimental Results}
\label{sec:experiments}

\vspace{-0.2cm}
\subsection{Experimental Setup}
\vspace{-0.2cm}
\textbf{Base Models.} All experiments use two DLMs with distinct designs:
LLaDA-8B-Instruct~\citep{nie2025llada}, trained from scratch with
block-wise inference, and Dream-7B-Instruct~\citep{ye2025dream}, adapted
from a Qwen2.5 ARLM via continued pretraining with non-block-wise sampling. Evaluating on
both ensures our method generalizes across diverse diffusion LM
architectures. Fine-tuning is performed without SFT on either backbone.
Details are in Appendix~\ref{sec:hypsetting}.

\textbf{Datasets.} We evaluate EGSPO and EGSPO-SA on a range of benchmark tasks, including mathematical reasoning (GSM8K \citep{cobbe2021training}, MATH500 \citep{lightman2023let}), logical reasoning  (Sudoku and Countdown \citep{ye2024beyond}), and coding (MBPP \citep{austin2021program}, HumanEval \citep{chen2021evaluating}). Across all benchmarks, we adopt the same terminal reward functions as used in prior work; detailed reward definitions are provided in Appendix~\ref{sec:dataset&reward}.


\textbf{Baselines.} We evaluate our algorithms against existing RL post-training methods for DLMs on both base models. The training and evaluation protocols of these baselines differ from ours in several aspects, which we detail below.

\textit{LLaDA experiments:} For mathematical and logical reasoning
tasks, we compare against d1~\citep{zhao2025d1}, d2~\citep{wang2025d2},
wd1~\citep{tang2025wd1}, SPG~\citep{wang2025spg}, and
ESPO~\citep{ou2026principled}. For coding tasks, only d1 and ESPO report
these results. In the tables, we report the numbers as published in the d1, d2, wd1, and SPG papers, all of which share our training protocol: a generation length of 256, two-token decoding per denoising step, and LoRA with the same rank. This is the standard training setup in the RL for DLM literature. For ESPO, we evaluate results by ourselves using their publicly released checkpoint under this unified protocol, since their reported
results use 1 token/step decoding for GSM8K, MATH500, and coding, and
1-shot evaluation for Sudoku. 1 token/step decoding bypasses the joint
probability learning that makes multi-token diffusion generation promising and 
non-trivial, and thus less appealing for pure diffusion-style language modeling. Full details on training and evaluation protocol differences
across all baselines are in Appendix~\ref{sec:baselines}.

\textit{Dream experiments:} We only compare against ESPO, as no other RL post-training baselines for  DLMs report results on this model. Since ESPO was trained and evaluated with 1-token/step decoding for GSM8K and MATH500, we also report our results with 1-token/step evaluation, even though our training used 2-token/step decoding.

\begin{table*}
\caption{Results on reasoning benchmarks for
\textbf{LLaDA-8B-Instruct.} We report accuracy at different
generation lengths and the best performance across them. d2 reports only the best performance; wd1 does not
report evaluation at length 128. For Sudoku, $^\dagger$ denotes $3$-shot
evaluation, whereas all other entries are $0$-shot. ESPO results are
re-evaluated from the publicly released ESPO checkpoints under the uniform
protocol shared by all other entries in the table, since ESPO's original
protocol uses 1-shot evaluation on Sudoku and decodes 1 token/step on GSM8K and MATH500; their originally published numbers
under that protocol are shown in \textcolor{gray}{gray}  for completeness.
}
\vspace{-0.1cm}
\label{tab:main_results}
\centering
\resizebox{\textwidth}{!}{
\begin{tabular}{lcccccccccccccccc}
\toprule
& \multicolumn{4}{c}{\textbf{Sudoku}}
& \multicolumn{4}{c}{\textbf{Countdown}}
& \multicolumn{4}{c}{\textbf{GSM8K}} 
& \multicolumn{4}{c}{\textbf{MATH500}} 
 \\
\cmidrule(lr){2-5}\cmidrule(lr){6-9}\cmidrule(lr){10-13}\cmidrule(lr){14-17}
\textbf{Model / Seq. Len.} 
& 128 & 256 & 512 & Best
& 128 & 256 & 512 & Best
& 128 & 256 & 512 & Best
& 128 & 256 & 512 & Best \\
\midrule

LLaDA-8B-Instruct
& 11.7 & 6.7 & 5.5 & 11.7
& 20.7 & 19.5 & 16.0 & 20.7
& 68.7 & 76.7 & 78.2 & 78.2
& 26.0 & 32.4 & 36.2 & 36.2 \\

d1
& 22.1 & 16.7 & 9.5 & 22.1
& 34.8 & 32.0 & 42.2 & 42.2
& 73.2 & 81.1 & 82.1 & 82.1
& \underline{33.8} & {38.6} & {40.2} & 40.2 \\

wd1
& - & 76.4 & 62.8 & 76.4
& - & 51.2 & 46.1 & 51.2
& - & 80.8 & 82.3 & 82.3
& - & 34.4 & 39.0 & 39.0
 \\

SPG
& 82.9{$^\dagger$} & 94.0{$^\dagger$} & 93.1{$^\dagger$} & 94.0{$^\dagger$}
& 68.8 & 71.5 &  70.3 & 71.5 
& \underline{78.5} & \textbf{86.1}  & 84.5 & \textbf{86.1}
& {33.4} & \underline{40.0} &  \textbf{41.8} & \textbf{41.8}
  \\

ESPO
& 72.9 & 58.6 & 54.5 & 72.9
& 81.6 & 82.0 & 79.3 & 82.0
& 76.7 & 79.9 & 80.8 & 80.8
& 32.0 & 34.6 & 39.4 & 39.4 \\

\textcolor{gray}{ESPO}
& \textcolor{gray}{92.7} & \textcolor{gray}{84.7} & \textcolor{gray}{80.5} & \textcolor{gray}{92.7}
& \textcolor{gray}{81.6} & \textcolor{gray}{82.0} & \textcolor{gray}{79.3} & \textcolor{gray}{82.0}
& \textcolor{gray}{80.0} & \textcolor{gray}{82.3} & \textcolor{gray}{83.7} & \textcolor{gray}{83.7}
& \textcolor{gray}{36.0} & \textcolor{gray}{39.0} & \textcolor{gray}{43.4} & \textcolor{gray}{43.3} \\

d2
&  &  &  & 91.9
&  &  &  & 56.6
&  &  &  & 85.0
&  &  &  & \underline{41.6} \\

\midrule
EGSPO (Ours)
& \underline{93.3} & \underline{93.6} & \underline{89.1} & \underline{93.6}
& \underline{82.0} & \underline{81.6} & \underline{82.4} & \underline{82.4}
& \textbf{80.7} & \underline{86.0} & \textbf{86.1} & \textbf{86.1}
& 32.2 & 37.8 & 39.0 & 39.0 \\

EGSPO-SA (Ours)
& \textbf{93.7} & \textbf{94.3} & \textbf{93.4} & \textbf{94.3}
& \textbf{84.0} & \textbf{85.2} & \textbf{83.2}  & \textbf{85.2}
& 77.0 & 84.6 & \underline{85.0} & \underline{85.0}
& \textbf{35.0} & \textbf{40.4} & \underline{41.6} &\underline{41.6} \\

\bottomrule
\end{tabular}
}
\vspace{-0.3cm}
\end{table*}

\begin{table*}
\caption{ Results on reasoning benchmarks for \textbf{Dream-7B-Instruct}.
We report accuracy at different generation lengths and the
best performance across them. For GSM8K and MATH500, all models are evaluated using 1 token/step decoding to match ESPO's protocol; ESPO is also trained with
1 token/step decoding, whereas our method is trained with 2 token/step decoding. For Sudoku and Countdown, both training and evaluation use 2 token/step decoding for all methods.}
\vspace{-0.25cm}

\label{tab:dream_results}
\centering
\resizebox{\textwidth}{!}{%
\begin{tabular}{lcccccccccccccccc}
\toprule
\multirow{2}{*}{\textbf{Model / Seq Len}} & \multicolumn{4}{c}{\textbf{Sudoku}} & \multicolumn{4}{c}{\textbf{Countdown}} & \multicolumn{4}{c}{\textbf{GSM8K}} & \multicolumn{4}{c}{\textbf{MATH500}} \\
\cmidrule(lr){2-5} \cmidrule(lr){6-9} \cmidrule(lr){10-13} \cmidrule(lr){14-17}
 & 128 & 256 & 512 & Best & 128 & 256 & 512 & Best & 128 & 256 & 512 & Best & 128 & 256 & 512 & Best \\
\midrule
\textbf{Dream}    & 9.3  & 2.1  & 14.0 & 14.0  & 8.5  & 7.8  & 17.4 & 17.4  & 75.8 & 81.3 & 80.7 & 81.3  & 38.2 & 45.7 & 48.0 & 48.0 \\
 ESPO            & {71.7} & {72.3} & {71.3} & {72.3} & {68.8} & {66.8} & {64.8} & {68.8} & \textbf{79.6} & \underline{82.3} & 82.0 & \underline{82.3} & 40.3 & \underline{47.4} & \underline{50.3} & \underline{50.3} \\
\midrule
 EGSPO (ours)    & \underline{86.4} & \textbf{90.3} & \underline{87.0} & \underline{90.3} & \underline{76.6} & \underline{73.4} & \underline{72.3} & \underline{76.6} & 77.4 & 82.1 & \underline{82.3} & \underline{82.3} & \textbf{42.4} & \textbf{49.2} & \textbf{50.4} & \textbf{50.4} \\
 EGSPO-SA (ours) & \textbf{90.6} & \textbf{90.3} & \textbf{89.2} & \textbf{90.6} & \textbf{79.7} & \textbf{82.4} & \textbf{78.5} & \textbf{82.4} & \underline{77.6} & \textbf{84.9} & \textbf{86.1} & \textbf{86.1} & \underline{41.2} & 47.2 & 47.8 & 47.8 \\
\bottomrule
\end{tabular}%
}
\vspace{-0.25cm}
\end{table*}

\begin{figure*}[t]
    \centering
    \begin{subfigure}[b]{0.32\textwidth}
        \includegraphics[width=\textwidth]{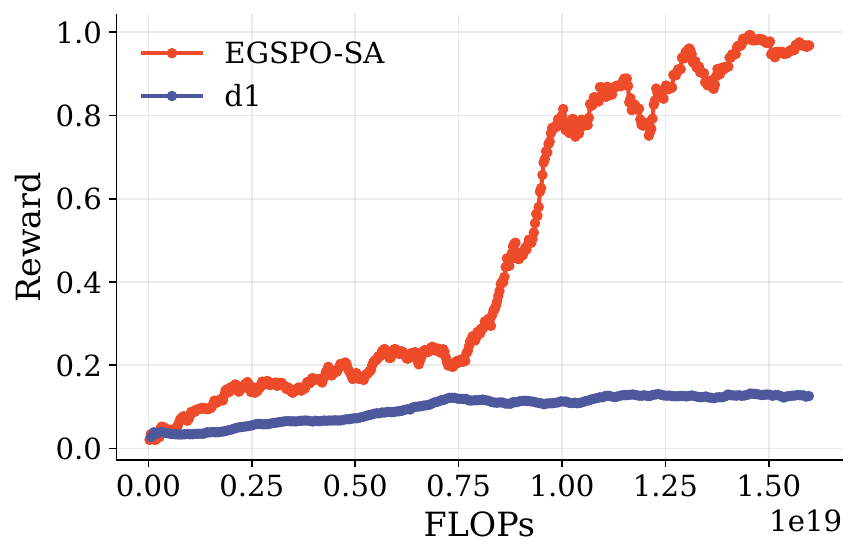}
        \caption{Reward vs.\ FLOPs}
    \end{subfigure}
    \hfill
    \begin{subfigure}[b]{0.32\textwidth}
        \includegraphics[width=\textwidth]{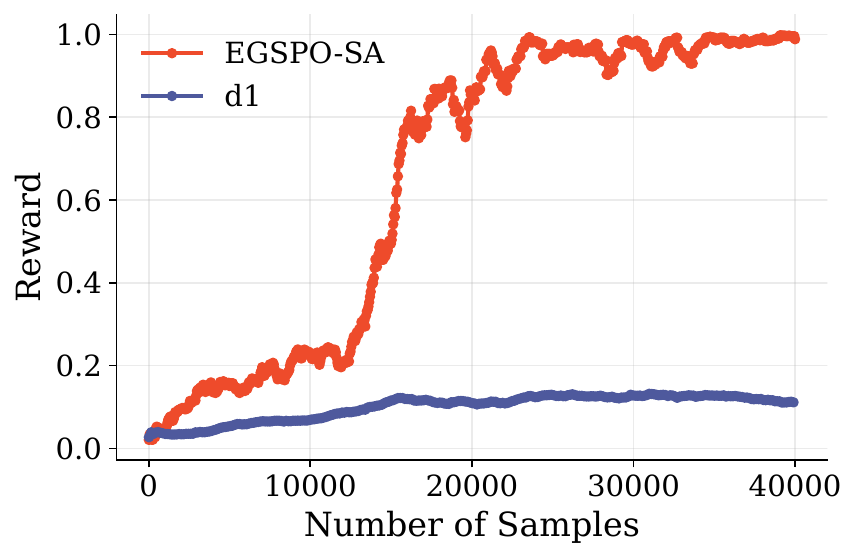}
        \caption{Reward vs.\ Samples}
    \end{subfigure}
    \hfill
    \begin{subfigure}[b]{0.32\textwidth}
        \includegraphics[width=\textwidth]{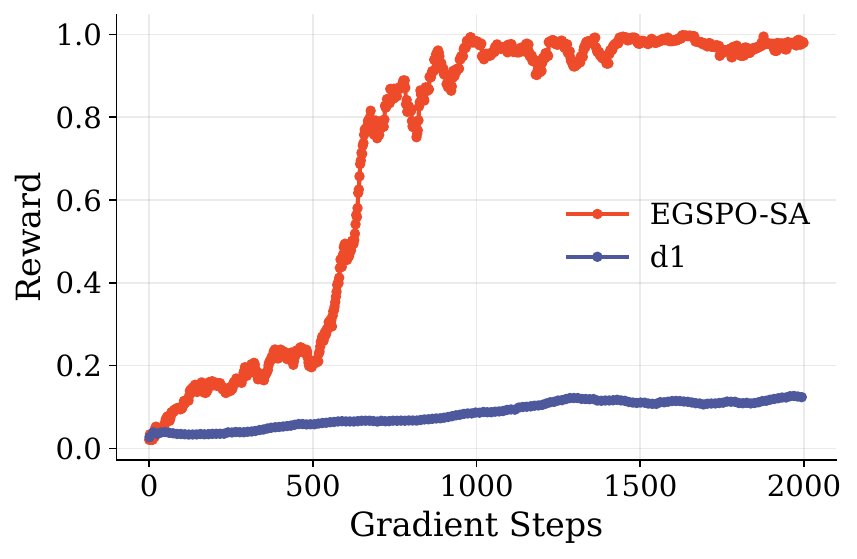}
        \caption{Reward vs.\ Gradient Steps}
    \end{subfigure}
    \vspace{-0.25cm}
\caption{
    Compute efficiency comparison between EGSPO-SA and d1 on Sudoku.
    \revision{\textbf{(a)} Total estimated FLOPs across 8 GPUs, including rollout generation, reference-policy evaluation, and policy forward/backward passes.}
    \textbf{(b)} Cumulative prompt--completion pairs seen during training.
    \textbf{(c)} Optimizer updates (accounting for gradient accumulation).
}
    \vspace{-0.5cm}
    \label{fig:compute_efficiency}
\end{figure*}

\begin{figure}[t]
\small 
    \centering
    \hfill
    \begin{subfigure}{0.3\textwidth}
        \centering
        \includegraphics[width=0.95\linewidth]{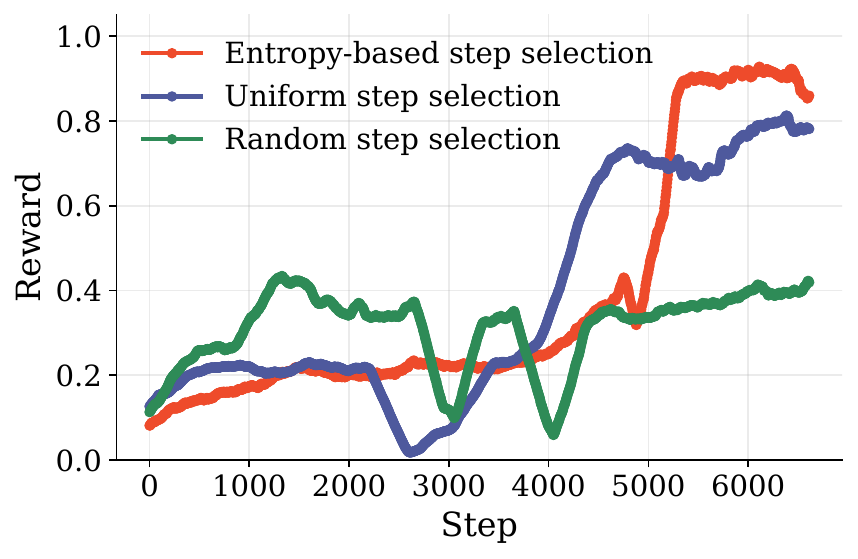}
        \caption{}
        \label{fig:ablation_entropy_selection}
    \end{subfigure}
    \hfill
    \begin{subfigure}{0.3\textwidth}
        \centering
        \includegraphics[width=0.95\linewidth]{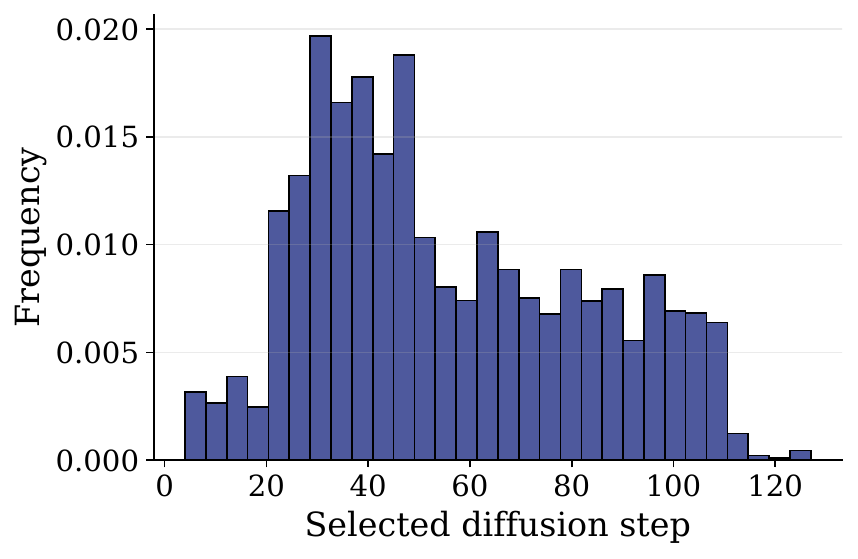}
        \caption{}
        \label{fig:ablation_Hist}
    \end{subfigure}
    \hfill 
    \begin{subfigure}{0.3\textwidth}
        \centering
        \includegraphics[width=\linewidth]{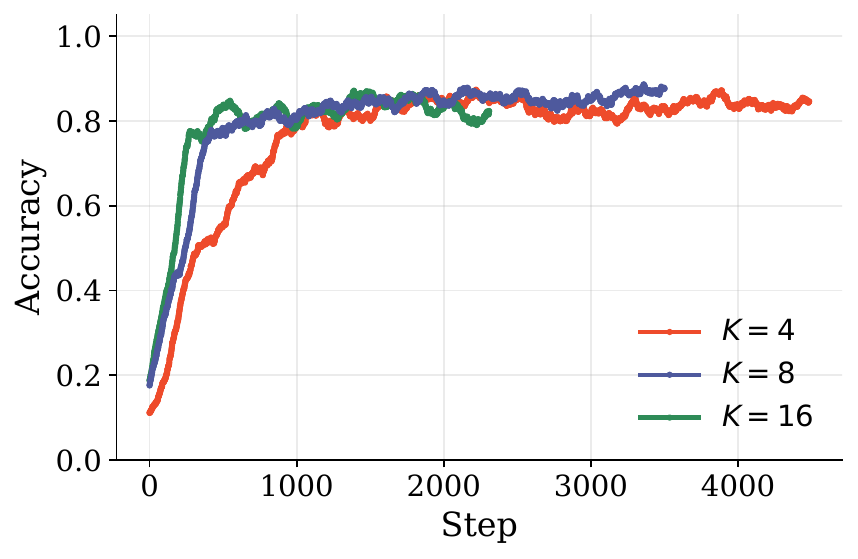}
        \caption{}
        \label{fig:ablation_K}
    \end{subfigure}
    \vspace{-0.25cm}
    \caption{ Ablation studies analyzing (a) entropy-guided step selection (EGSPO) versus uniform and random step selection, (b) the distribution of selected denoising steps, and (c) performance as a function of the step budget $K$.}
    \vspace{-0.6cm}
    \label{fig:ablation_study}
\end{figure}

\subsection{Main Results}
\label{sec:main_results}

Tables~\ref{tab:main_results}, \ref{tab:dream_results}, and~\ref{tab:coding_results} report performance across all tasks and generation lengths on both models. Training curves are provided in
Appendix~\ref{app:training_dynamics}, and qualitative examples for
each dataset are shown in Appendix~\ref{app:Qualitative}.

\textbf{Logical reasoning:} EGSPO and EGSPO-SA achieve the largest gains on Sudoku and Countdown, consistently across both backbones. EGSPO-SA outperforms EGSPO  across all generation lengths on both tasks, with the gap most pronounced on Sudoku. \revision{This pattern is consistent with the intuition in \cref{sec:EstimatingAdvantages}: intermediate states progressively fix parts of the solution in logical reasoning tasks, making greedy one-step completions a useful reference for evaluating the final rollout reward.} On LLaDA, EGSPO-SA achieves 94.3\% on Sudoku (vs.\ 94.0\% for SPG under the easier 3-shot setting, and 72.9\% for ESPO) and 85.2\% on Countdown (vs.\ 82.0\% for ESPO and 71.5\% for SPG), both under 0-shot evaluation. On Dream, EGSPO-SA reaches 90.6\% on Sudoku and 82.4\% on Countdown, surpassing ESPO by 18 and 14 percentage points respectively.

\textbf{Mathematical reasoning:} On GSM8K and MATH500, our algorithms  match or exceed existing approaches. On LLaDA, EGSPO achieves 86.1\% on GSM8K, matching SPG (86.1\%), and EGSPO-SA reaches 41.6\% on MATH500, matching d2 and within 0.2\% of SPG (41.8\%). The gap between EGSPO and EGSPO-SA is small on both tasks, which is consistent with the observation from \cref{sec:EstimatingAdvantages} that the sequence-level reward already provides a sufficient learning signal when task structure does not tightly couple denoising steps. \revision{On a few GSM8K configurations, EGSPO outperforms EGSPO-SA. One possible explanation is that greedy completions from intermediate reasoning states provide a less reliable reference for the final reward, so incorporating the stepwise term does not consistently improve the practical training update.}
On Dream, EGSPO-SA achieves 86.1\% on GSM8K, outperforming ESPO (82.3\%).

\begin{wraptable}{r}{0.51\textwidth}
\vspace{-0.5cm}
\caption{Results on coding benchmarks for \textbf{LLaDA-8B-Instruct.}
We report accuracy at different generation lengths and the best performance across them. ESPO uses full fine-tuning, whereas d1, EGSPO, and EGSPO-SA use LoRA-based fine-tuning. ESPO results are re-evaluated from the publicly released ESPO checkpoints under the uniform protocol shared by all other entries in the table, since ESPO's original protocol decodes 1 token/step on coding tasks; their originally published numbers under that protocol are shown in \textcolor{gray}{gray} for completeness.}

\label{tab:coding_results}
\centering
\resizebox{\linewidth}{!}{%
\begin{tabular}{lcccccccc}
\toprule
& \multicolumn{4}{c}{\textbf{HumanEval}}
& \multicolumn{4}{c}{\textbf{MBPP}} \\
\cmidrule(lr){2-5}\cmidrule(lr){6-9}
\textbf{Model}
& 128 & 256 & 512 & Best
& 128 & 256 & 512 & Best \\
\midrule
LLaDA & 27.4 & 35.3 & 37.8 & 37.8 & 36.2 & 41.2 & 40.4 & 41.2 \\
d1    & 31.1 & 32.9 & 37.8 & 37.8 & 40.5 & 44.7 & 42.8 & 44.7 \\
ESPO  & 25.0 & 31.7 & 36.0 & 36.0 & 49.0 & 48.6 & 49.0 & 49.0 \\
\textcolor{gray}{ESPO} & \textcolor{gray}{28.1} & \textcolor{gray}{42.1} & \textcolor{gray}{50.0} & \textcolor{gray}{50.0} & \textcolor{gray}{47.4} & \textcolor{gray}{44.6} & \textcolor{gray}{44.2} & \textcolor{gray}{47.4} \\
\midrule
EGSPO & \textbf{32.3} & \underline{40.2} & \underline{39.6} & \underline{40.2} & \underline{49.9} & \textbf{50.6} & \textbf{50.3} & \underline{50.6} \\
EGSPO-SA & \textbf{32.3} & \textbf{41.5} & \textbf{44.5} & \textbf{44.5} & \textbf{51.1} & \underline{48.7} & \underline{49.2} & \textbf{51.1} \\
\bottomrule
\end{tabular}%
}
\vspace{-0.8cm}
\end{wraptable}


\textbf{Coding:} Table~\ref{tab:coding_results} show results on HumanEval and MBPP for LLaDA-8B-Instruct. Both EGSPO and EGSPO-SA substantially outperform d1 and ESPO under the matched 2 token/step decoding protocol. EGSPO-SA achieves 51.1\% on MBPP (vs.\ 44.7\% for d1 and 49.0\% for ESPO) and 44.5\% on HumanEval (vs.\ 37.8\% for d1 and 36.0\% for ESPO), demonstrating that entropy-guided step selection is effective for program synthesis, where token-level uncertainty varies substantially across the denoising trajectory.


\vspace{-0.3cm}
\subsection{Ablations and Additional Experiments}

\textbf{Compute Efficiency:} Figure~\ref{fig:compute_efficiency} compares EGSPO-SA and d1 on Sudoku along three compute axes (FLOPs, prompt--completion pairs, and
optimizer updates) using LLaDA-8B-Instruct as the base model. Across
all three, EGSPO-SA converges to near-perfect reward while d1 plateaus
below $0.2$, showing that EGSPO-SA is more efficient in compute, data,
and optimization. \revision{FLOPs are estimated using $2NL$ per forward pass and $6NL$ per forward-and-backward pass, where $N$ is the model parameter count and $L$ is the input sequence length. Old-policy log-probabilities are cached during generation and incur no additional forward passes.} \revision{An analytical cost comparison with additional baselines is provided in Appendix~\ref{app:compute_comparison}.}



\textbf{Effect of Entropy-Based Step Selection:} Under a fixed budget of $K$ denoising steps per trajectory, we compare
\textit{uniform} (evenly spaced), \textit{random} (randomly
sampled), and \textit{entropy-based} (top-$K$ by average unmasking
token entropy) selection. \revision{For a fixed $K$, all three strategies use the same number of rollout, reference-policy, and policy forward/backward passes; they differ only in which denoising steps are selected. Entropy scores are obtained from the rollout predictions without additional model forward passes.} As shown in Figure~\ref{fig:ablation_entropy_selection},
EGSPO achieves higher final performance and converges more rapidly under
matched compute, as it concentrates updates on high-uncertainty timesteps,
consistent with the analysis in Section~\ref{sec:entropy_step_selection}.
Uniform consistently outperforms random as training progresses, indicating
that covering the trajectory provides more stable updates than stochastic
sampling.

\textbf{Distribution of Selected Denoising Steps:}
Figure~\ref{fig:ablation_Hist} shows the distribution of timesteps
chosen by entropy-based selection. Contrary to the naive expectation that
early, heavily masked steps would dominate, selected steps span the
diffusion horizon with a modest concentration on intermediate timesteps.
In early steps, the model is confident across the unmasked-token choices;
intermediate steps are where partial structure has formed but multiple
completions remain plausible, yielding the highest token-level uncertainty;
late steps are largely deterministic and rarely selected. Entropy-based
selection thus targets timesteps where the model is genuinely uncertain
about token assignments, not simply those that are more masked.


\begin{wraptable}{r}{0.42\textwidth}
\vspace{-0.8em}
\caption{ GSM8K accuracy (EGSPO, LLaDA) for different step budgets $K$ and generation length}
\label{tab:k_ablation}
\centering
\resizebox{\linewidth}{!}{%
\begin{tabular}{lccccc}
\toprule
$K$ & \small{Approx. Grad. Steps} & 128 & 256 & 512 \\
\midrule
4  & 4400 & 76.5 & 83.4 & 83.6 \\
8  & 3600 & \textbf{80.7} & \textbf{86.0} & \textbf{86.1} \\
16 & 2400 & 78.2 & 83.4 & 83.6 \\
\bottomrule
\end{tabular}}
\vspace{-0.5em}
\end{wraptable}

\textbf{Effect of Step Budget $K$:} \revision{We compare step budgets under a fixed wall-clock budget of 48 hours on 8 NVIDIA H200 GPUs (384 GPU-hours per run), including rollout generation, reference-policy evaluation, and policy optimization. Increasing $K$ requires gradient computation at more denoising steps per batch, reducing the number of optimizer updates completed within this budget.} This gives approximately $4400$, $3600$, and $2400$ total gradient steps for $K=4,8,16$ respectively. Table~\ref{tab:k_ablation} shows that $K=8$ achieves the best GSM8K accuracy across all generation lengths: too few gradient terms ($K=4$) gives insufficient coverage of the denoising trajectory, while too few total steps ($K=16$) under-optimizes the policy. Entropy-guided selection concentrates the $K$ budget on the most informative steps, making this tradeoff favorable at moderate $K$; training curves are shown in Figure~\ref{fig:ablation_K}.


\textbf{Advantage weight $\lambda$.} We additionally ablate the choice of constant vs.\ step-varying $\lambda$ in Appendix~\ref{app:lambda-ablation}, finding that a $\lambda$ fixed within each episode (no denoising-step dependence) matches a diffusion time-step dependent weight $\lambda_t$, which justifies this simpler within-episode choice. 
\revision{We select $K$, the initial value of $\lambda$, and whether to anneal $\lambda$ across training using training reward curves: candidate configurations are trained for 1000--2000 steps, and the selected configuration is then trained to completion. Per-task settings are listed in Appendix~\ref{sec:hypsetting}. We retain the training and test splits used by d1, without introducing a validation split. Test-set performance is not used for hyperparameter selection.}

\section{Conclusion}
We introduced a principled RL framework for diffusion language models
that makes denoising structure explicit. By formulating masked diffusion
generation as a finite-horizon MDP, we derived the policy gradient
that decomposes along the diffusion trajectory and yields a well-defined
notion of stepwise advantages, without evaluating intractable
sequence-level likelihoods. We proposed entropy-guided step selection
for adaptive compute allocation and lightweight one-step completions for
intermediate advantage estimation, enabling efficient stepwise policy
optimization with theoretical guarantees. Empirical results demonstrate
consistent improvements over existing RL post-training methods for DLMs
across coding, logical, and mathematical reasoning benchmarks. 

An interesting future work is to extend our algorithms to self-correction and remasking-capable DLMs, where tokens may be revised after being
decoded. This requires revisiting both the action space and the policy gradient
derivation under non-monotonic trajectories.

\newpage
\section*{Acknowledgements and Disclosures}

Portions of this research were conducted using advanced computing resources provided by Texas A\&M High Performance Research Computing. This work was supported in part by the National Science Foundation (NSF) under grants CNS-2148354 and FuSe2-2425399, including support under CNS-2148354 from NSF, federal agencies, and industry partners through the Resilient \& Intelligent NextG Systems (RINGS) program; and by the U.S. Army Combat Capabilities Development Command (DEVCOM) under grant W911NF2520046. Any opinions, findings, conclusions, or recommendations expressed in this material are those of the authors and do not necessarily reflect the views of the sponsoring agencies or industry partners.

\bibliography{neurips_contents/ref}
\bibliographystyle{plainnat}




\appendix

\newpage
\section{Theoretical Results}

\label{app:theoretical_results}


\textbf{Proposition~\ref{thm:pg_thm}} (Recalled). For a policy $\pitheta$ and the objective $J(\btheta) = \bE_{\vq, \vx \sim \pitheta(\cdot \mid \vq)}[r(\vx_0, \vq)]$, the policy gradient is given by \eqref{eq:pg_adv}.

\begin{proof}
\label{proof:pg_thm}
    Writing $\vx = \vx_{0:T-1}$, we can evaluate directly
    \begin{align*}
        \nabla_\btheta J(\btheta) 
        &= \nabla_\btheta  \bE_\vq \left[ \sum_{\vx} r(\vx_0, \vq) \pitheta(\vx \mid \vq) \right] =\bE_\vq \left[ \sum_\vx  r(\vx_0, \vq) \nabla_\btheta \pitheta(\vx \mid \vq) \right] \\
        &=\bE_\vq \left[ \sum_\vx  r(\vx_0, \vq) \pitheta(\vx \mid \vq) \nabla_\btheta \log  \pitheta(\vx \mid \vq) \right] =\bE_{\vq, \vx \sim \pitheta(\cdot \mid \vq)}\left[ r(\vx_0, \vq) \nabla_\btheta \log \pitheta(\vx \mid \vq) \right].
    \end{align*}
    By definition, since $\pitheta(\vx \mid \vq) = \pitheta(\vx_{0:T-1} \mid \vq) = \prod_{t=0}^{T-1} \pitheta(\vx_t \mid \vx_{t+1}, \vq)$, we have $\nabla_\btheta J(\btheta) = \sum_{t=0}^{T-1} \nabla_\btheta J_t(\btheta)$, where the per step policy gradient is defined as $\nabla_\btheta J_t(\btheta) = \bE_{\vq, \vx\sim \pitheta}[r(\vx_0, \vq) \nabla_\btheta \log \pitheta(\vx_t \mid \vx_{t+1}, \vq)]$. The term inside the expectation is a random variable that depends on $\vx_t, \vx_{t+1}, \vx_0, \vq$. For any action $\vx_t$ independent baseline, in particular, the random variable $V_{t+1}^\pitheta(\vx_{t+1})$, we have
    \begin{align*}
        &~\bE_{\vx_{t} \sim \pitheta(\cdot \mid \vx_{t+1})}[V_{t+1}^\pitheta(\vx_{t+1}) \nabla_\btheta \log \pitheta(\vx_t \mid \vx_{t+1})] = V_{t+1}^\pitheta(\vx_{t+1}) \bE_{\vx_{t} \sim \pitheta(\cdot \mid \vx_{t+1})}[\nabla_\btheta \log \pitheta(\vx_t \mid \vx_{t+1})] \\
        &~=V_{t+1}^\pitheta(\vx_{t+1}) \sum_{\vx_t} \pitheta(\vx_t \mid \vx_{t+1}) \nabla_\btheta \log \pitheta(\vx_t \mid \vx_{t+1}) =V_{t+1}^\pitheta(\vx_{t+1}) \sum_{\vx_t} \nabla_\btheta \pitheta(\vx_t \mid \vx_{t+1}) \\
        &=V_{t+1}^\pitheta(\vx_{t+1}) \nabla_\btheta \left( \sum_{\vx_t}\pitheta(\vx_t \mid \vx_{t+1})\right) = V_{t+1}^\pitheta(\vx_{t+1}) \nabla_\btheta (1) = 0.
    \end{align*}
    Therefore, $A_t^\pitheta = r(\vx_0, \vq) - V_{t+1}^\pitheta(\vx_{t+1}, \vq)$ satisfies $ \nabla_\btheta J_t(\btheta) = \bE_{\vq, \vx\sim\pitheta(\cdot \mid \vq)}[A_t^\pitheta \nabla_\btheta \log \pitheta(\vx_t \mid \vx_{t+1})]$.

\end{proof}


\begin{proposition}[Entropy Bound]
\label{prop:entropy_bound}
    Let $\pitheta$ be a softmax policy with logits $\gtheta^i$ for $i$ over a finite set, i.e., $\pi_{\btheta, i} \triangleq \frac{\exp(g_{\btheta,i})}{\sum_{i'}\exp(g_\btheta^{i'})}$. Let $A_i$ be constants such that $\norm{A}_\infty\triangleq \max_i \abs{A_i}$ is finite. Let $J(\btheta) \triangleq \bE_{i\sim \pitheta}[A_i]$ and suppose $\norm{\nabla_{\btheta} g_\btheta} \triangleq \max_{i,j} \norm{\nabla_\btheta (g_{\btheta,i}-g_{\btheta,j})}$ is bounded. Then, $\nabla_\btheta J(\btheta) = \bE_{i\sim \pitheta}[A_i \nabla_\btheta \log \pi_{\btheta, i}]$ satisfies
    \begin{align}
        \norm{\nabla_\btheta J(\btheta)} \le H(\pitheta) \norm{A}_\infty \norm{\nabla_\btheta g_\btheta}, \label{eq:grad_score_bound} 
    \end{align}
    where $H(\pitheta) \triangleq \bE_{i\sim \pitheta}[-\log \pi_{\btheta, i}]$ is the entropy of $\pitheta$.
\end{proposition}

\begin{proof} Starting with the gradient of the objective $J$, we have $\norm{\nabla_\btheta J(\btheta)} \le \bE_{i\sim \pitheta}[\abs{A_i} \norm{\nabla_\btheta \log \pi_{\btheta, i}}] \le \norm{A}_\infty \bE_{i\sim \pitheta}[ \norm{\nabla_\btheta \log \pi_{\btheta, i}}]$. Now, for any $i$, using chain-rule and the derivative of softmax policy, we have $\nabla_\btheta \log \pi_{\btheta, i}  = (1-\pi_{\btheta, i}) \nabla_\btheta g_{\btheta,i} -\sum_{j\neq i} \pi_{\btheta, j}\nabla_\btheta g_{\btheta,j} = \sum_{j\neq i} \pi_{\btheta, j}\left(\nabla_\btheta g_{\btheta,i} - \nabla_\btheta g_{\btheta,j}\right)$. Therefore, by triangle inequality, we obtain
\begin{align}
    \norm{\nabla_\btheta \log \pi_{\btheta, i}} \le \sum_{j\neq i} \pi_{\btheta, j}\norm{\nabla_\btheta(g_{\btheta,i}-g_{\btheta,j})} &\le \left(\sum_{j\neq i} \pitheta^j\right) \left(\max_{i,j} \norm{\nabla_\btheta(g_{\btheta,i}-g_{\btheta,j}))} \right)  \\
    &~ = (1-\pi_{\btheta, i}) \norm{\nabla_\btheta g_\btheta} \le \left(-\log \pi_{\btheta, i} \right) \norm{\nabla_\btheta g_\btheta}, \label{eq:score_bound}
\end{align}
where the last step follows from $1 - x \le -\log x$ for $x\in (0, 1)$. Using the above bound in the first inequality gives the result.
\end{proof}

\begin{proof}[Proof of Lemma~\ref{lem: error_bound_using_entropy}]
\label{proof: error_bound_using_entropy}
By Proposition \ref{prop:entropy_bound}, we obtain the following bound on the error in terms of the per-step entropy:
\begin{align*}
\Delta_S &= \norm{\nabla_\btheta J(\btheta) - \nabla_\btheta J_S(\btheta)}
    = \norm{\sum_{t \notin S} \nabla_\btheta J_t(\btheta)} \le \sum_{t \notin S} \norm{\nabla_\btheta J_t(\btheta)} \\
    &\le \sum_{t \notin S} H(\pitheta^{t\mid t+1}) \norm{A_t^{\pitheta}}_\infty \norm{\nabla_\btheta g_{\btheta}^{t\mid t+1}},
\end{align*}
where $g_\btheta$ are the logits from which the policy $\pitheta$ is constructed by taking softmax. Bounding $\norm{A_t^{\pi_\btheta}}_\infty \le \max_{t'} |A_{t'}^{\pi_\btheta}|$ and $\norm{\nabla_\btheta g_\btheta^{t\mid t+1}} \le B$ uniformly over $t$, we obtain $\Delta_S \le B\!\left(\max_t |A_t^{\pi_\btheta}|\right)\!\sum_{t \notin S} H(\pitheta^{t\mid t+1})$.
\end{proof}

\section{Diagnostic Evaluation of Entropy-Guided Step Selection}
\label{app:gradient_diagnostic}

\revision{To assess whether entropy-guided selection preserves the direction of the full denoising-step update, we compute gradients at all $T=128$ steps on a frozen LLaDA-8B-Instruct model, using 128 completions per task. We compare the cosine similarity between the update formed from $K$ selected steps and the full $T$-step update. Computing all step gradients is used only for this diagnostic and is not required during training.}

\begin{table}[ht]
\centering
\small
\caption{\revision{Cosine similarity between selected-step and full-step updates. Entropy-guided selection gives higher alignment than uniform and random selection across all evaluated tasks and budgets.}}
\label{tab:gradient_diagnostic}
\begin{tabular}{llccc}
\toprule
Task & $K$ & Entropy & Uniform & Random \\
\midrule
Sudoku    & 4 & 0.466 & 0.182 & 0.194 \\
          & 8 & 0.603 & 0.258 & 0.276 \\
Countdown & 4 & 0.447 & 0.187 & 0.202 \\
          & 8 & 0.583 & 0.288 & 0.298 \\
GSM8K     & 4 & 0.576 & 0.120 & 0.150 \\
          & 8 & 0.720 & 0.253 & 0.237 \\
MATH500   & 4 & 0.510 & 0.181 & 0.169 \\
          & 8 & 0.654 & 0.252 & 0.251 \\
\bottomrule
\end{tabular}
\end{table}

\revision{Table~\ref{tab:gradient_diagnostic} provides empirical support for entropy-guided selection: its selected-step updates align more closely with the full-step updates than those obtained by uniform or random selection. These results complement the upper-bound analysis in Lemma~\ref{lem: error_bound_using_entropy}.}

\section{Datasets and Reward Functions}\label{sec:dataset&reward}

We largely follow the experimental protocol of d1 \citep{zhao2025d1}, adopting the same reward formulations and train--test splits to ensure comparability.

\subsection{GSM8K.}
We use the training split of the GSM8K dataset for reinforcement learning and evaluate on the official test split. Rewards follow the Unsloth-style formulation and consist of five additive components:

\begin{itemize}
    \item an XML structure reward that assigns $+0.125$ for each correctly placed formatting tag, with small penalties for extraneous content appearing after closing tags;
    \item a soft format reward of $+0.5$ for outputs matching the pattern
    \texttt{<reasoning>...</reasoning><answer>...</answer>}
    \item a strict format reward of $+0.5$ for exact adherence to the expected structure, including correct line breaks;
    \item an integer-answer reward of $+0.5$ if the predicted answer is a valid integer;
    \item a correctness reward of $+2.0$ when the predicted answer matches the ground-truth solution.
    \end{itemize}

\subsection{MATH500.}
For MATH500, we train on the training split and evaluate on the test split. The reward function consists of formatting and correctness components:

\begin{itemize}
    \item a format reward of $1.0$ if \texttt{<answer></answer>} tags are present and the answer is enclosed in a \texttt{\textbackslash boxed} expression;
    \item a format reward of $0.75$ if \texttt{<answer></answer>} tags are present without a boxed expression;
    \item a format reward of $0.5$ if a boxed expression is present without answer tags;
    \item a format reward of $0.25$ if neither answer tags nor a boxed expression are present;
    \item a correctness reward of $+2.0$ when the boxed answer exactly matches the ground-truth solution.
\end{itemize}

\subsection{Sudoku.}
For the $4\times4$ Sudoku task, we use a publicly available dataset of one million synthetically generated puzzles.\footnote{\url{https://github.com/Black-Phoenix/4x4-Sudoku-Dataset}} The dataset was generated using the Arel solver. For evaluation, we randomly sample 256 puzzles using the same generator. The reward is computed as the fraction of correctly filled cells among positions that were initially empty in the input puzzle, focusing evaluation on reasoning performance rather than copying pre-filled values.

\subsection{Coding.}
For coding experiments, we train on the KodCode-Light-RL-10k dataset. The reward function consists of three components. First, an XML structure reward identical to GSM8K is used, with an additional $+0.5$ bonus when the generated program is enclosed within answer tags; outputs not wrapped in \texttt{python} code blocks receive zero structural reward. Second, a correctness score is computed using unit tests, where the reward corresponds to the fraction of tests passed rather than a binary success signal. Finally, a safety constraint assigns a reward of $0$ if the generated code imports restricted modules, including \texttt{os}, \texttt{sys}, \texttt{shutil}, \texttt{subprocess}, \texttt{socket}, \texttt{psutil}, \texttt{ctypes}, \texttt{pathlib}, \texttt{builtins}, or \texttt{import}.

\begin{algorithm}[tbp]
\caption{Entropy-Guided Stepwise Policy Optimization with Stepwise Advantages (EGSPO-SA)}
\label{alg:epsa}
\begin{algorithmic}[1]
\REQUIRE Policy $\pitheta$, reference policy $\pi_{\text{ref}}$, reward $r$,
prompt distribution $\mathcal{D}$, denoising steps $T$,
step budget $K$, completions $G$, advantage weight $\lambda$,
KL weight $\beta$
\STATE Initialize $\btheta \gets \btheta_{\text{ref}}$

\WHILE{not converged}
    \STATE Sample prompt $\vq \sim \mathcal{D}$
    \STATE Sample $G$ diffusion trajectories
    $\vx^{j}_{0:T} \sim \pithetaold(\cdot \mid \vq)$ for $j=1, \dots, G$
    \STATE Compute per-step entropies
    $H_t^j$ of $\pithetaold(\cdot \mid \vx_{t+1}^j)$
    \STATE Select step subset
    $S^j \gets \text{top}[K](t: H_t^j)$

    \STATE \revision{Compute reference-policy probabilities $\pi_{\rm ref}(\vx_t^j \mid \vx_{t+1}^j)$ for $j=1,\dots,G$ and $t\in S^j$.}

    \STATE Greedily complete trajectory from $\vx^{j}_{t+1}$:
    $\hat{\vx}^{j}_{0\mid t+1} \gets \text{GreedyComplete}(\vx^{j}_{t+1})$ (as per \cref{eq:greedy_completion})

    \STATE Estimate per-step coefficients per completion:
    $A^{j}_{t} \gets r(\vx^{(j)}_0, \vq) +\lambda(r(\vx^{(j)}_0, \vq) 
    - r(\hat{\vx}^{j}_{0\mid t+1}, \vq))$.

    \STATE Compute centered stepwise advantages: $\bar{A}_t^{j} \gets A^{j}_{t}
    - \frac{1}{G}\sum_{i=1}^G A^{i}_{t}$

    \STATE Compute importance ratios
    $\rho_t^{j} \gets
    \pitheta(\vx^{j}_t \mid \vx^{j}_{t+1}) /
          {\pithetaold (\vx^{j}_t \mid x^{j}_{t+1})}$

    \STATE Estimate KL term using $\pitheta$ and $\pi_{\rm ref}$ likelihoods using trajectories from $\pithetaold$

    \STATE Update $\pitheta$ using GRPO with stepwise advantages $\bar{A}_t^{j}$ and KL regularization weighted by $\beta$
    
\ENDWHILE

\STATE Return $\pitheta$
\end{algorithmic}
\end{algorithm}

\section{Analytical Compute Comparison}
\label{app:compute_comparison}

\revision{We complement the measured comparison in Figure~\ref{fig:compute_efficiency} with an analytical accounting of training cost. We express costs in forward-pass equivalents: one model forward pass costs one unit, and a forward-and-backward pass costs approximately three units. At a response length of 256 with two tokens decoded per step, generation requires 128 forward passes. Table~\ref{tab:compute_comparison} reports the additional update costs under the listed configurations, including reference- and old-policy evaluations where required.}

\begin{table}[ht]
\centering
\small
\caption{\revision{Analytical cost per generated answer in forward-pass equivalents. Totals include 128 generation passes. Here $\mu$ is the number of update iterations, $m$ and $M$ are Monte Carlo sample counts, and $K$ is our selected-step budget.}}
\label{tab:compute_comparison}
{\color{black}
\begin{tabular}{llccc}
\toprule
Method & Settings & Grad. passes & Update cost & Total \\
\midrule
SPG   & $m=2,\ \mu=4$ & 8  & $8\times3=24$       & 152 \\
Ours  & $K=8,\ \mu=1$ & 8  & $8\times3+8=32$     & 160 \\
wd1   & $\mu=12$      & 12 & $12\times3=36$      & 164 \\
d1    & $\mu=12$      & 12 & $12\times3+2=38$    & 166 \\
ESPO  & $M=2,\ \mu=8$ & 32 & $32\times3=96$      & 224 \\
\bottomrule
\end{tabular}
}
\end{table}

\revision{Generation accounts for most of the total cost for SPG, our method, wd1, and d1. Their total costs are therefore similar despite differences in update allocation. Our method uses eight selected trajectory steps and eight reference-policy forward passes, while old-policy log-probabilities are cached during generation. ESPO's published single-token decoding configuration for mathematical reasoning and coding instead requires 256 generation passes, increasing its total to 352 forward-pass equivalents. These estimates describe computational cost under the listed configurations, rather than measured wall-clock performance.}

\section{Details on Baselines}
\label{sec:baselines}

For LLaDA-8B-Instruct experiments, our method along with d1, wd1, and
SPG fine-tune LLaDA-8B-Instruct at a fixed generation length of 256
with two tokens decoded in parallel per denoising step, using LoRA
with the same rank. The remaining baselines deviate from this protocol
as detailed below.

\begin{itemize}
    \item \textbf{d2} trains and evaluates with task-specific generation
    lengths (e.g., 128 for Sudoku and Countdown, 512 for MATH500), so
    its reported numbers are not produced under a uniform decoding
    configuration across tasks.

    \item \textbf{ESPO} uses single-token decoding on GSM8K, MATH500,
    and coding tasks, and two-token decoding on Sudoku and Countdown,
    for both training and evaluation. It additionally uses full
    fine-tuning rather than LoRA on coding tasks, and reports a 1-shot
    result on Sudoku. To place ESPO on equal footing with the rest of
    our table, we take the publicly released ESPO checkpoints and
    re-evaluate them under our uniform protocol (two tokens decoded per
    step, and 0-shot evaluation) on all reasoning and coding benchmarks.

    \item \textbf{SPG} reports a 3-shot result on Sudoku, while all
    other entries in our table, including ours, are 0-shot. SPG's
    3-shot number is retained as published, since the authors did not
    release their checkpoints.
\end{itemize}

\section{Hyperparameter Settings and Implementation Details} \label{sec:hypsetting}

We follow prior diffusion-RL work for most hyperparameter choices.
Low-Rank Adaptation (LoRA) is applied with rank $r = 128$ and scaling
factor $\alpha = 64$. Training is conducted on 8 NVIDIA H200 GPUs.
Optimization is performed using Adam with $\beta_1 = 0.9$, $\beta_2 = 0.99$,
weight decay of $0.1$, and a learning rate of $3 \times 10^{-5}$.
Task-specific settings for batch size, number of completions per prompt,
step budget, gradient accumulation, and the advantage weight $\lambda$
are listed in Tables~\ref{tab:hyp_llada} and~\ref{tab:hyp_dream}.

For RL rollout, we use a generation length of $256$ tokens with $128$
diffusion steps during training. We adopt block-wise generation with a
block size of $32$, and two tokens are denoised in parallel at each step
for all tasks. For each prompt, we generate $G$ completions, and for each
completion we select the top $K$ denoising steps with the highest entropy
for policy updates. We use $K = 8$ for all tasks except coding, where $K = 16$: coding tasks have longer effective horizons (more tokens must be placed correctly in syntactically valid positions), so broader trajectory coverage from a larger $K$ outweighs the reduction in total gradient steps under the fixed compute budget. We use a sampling temperature of $0.9$ during training
for all tasks except Countdown, where a temperature of $0.2$ is used. At
inference time, we apply greedy decoding for all tasks.

At evaluation time, we report results using generation lengths of $128$,
$256$, and $512$ tokens. Models are evaluated once the reward curves
stabilize: coding tasks are trained for $4$k steps, mathematical reasoning
tasks for around $6$k steps, and logical reasoning tasks for $10$k steps.
For all tasks, we select the checkpoint with the highest average
performance across the evaluated generation lengths.

The advantage weight $\lambda$ varies by task and backbone; the settings are listed in Tables~7 and~8.

For Sudoku and GSM8K on LLaDA, we use a constant $\lambda$. For Countdown, MATH500, and most Dream tasks, we
find that initializing $\lambda$ at the values listed in
Tables~\ref{tab:hyp_llada}--\ref{tab:hyp_dream} and halving it every $500$
training steps further improves performance, analogous to standard
hyperparameter scheduling in deep RL (e.g., KL coefficient annealing).

\begin{table}[h]
\centering
\caption{Task-specific hyperparameters for \textbf{LLaDA-8B-Instruct}. $B$: batch size (prompt-completion pairs per update), $G$: completions per prompt, $K$: step budget.}
\label{tab:hyp_llada}
\small
\begin{tabular}{lccccc}
\toprule
\textbf{Task} & $B$ & $G$ & $K$ & Grad.\ Accum. & $\lambda$ \\
\midrule
Sudoku    & 48 & 8 & 8 & 2 & 1 (constant) \\
Countdown & 96 & 8 & 8 & 2 & 1 (halved every 500 steps) \\
GSM8K     & 48 & 8 & 8 & 8 & 0.2 (constant) \\
MATH500   &  96  &  8 &  8 &  8 & 1 (halved every 500 steps) \\
Coding    &  96  & 8  & 16  &  4 & 1 (halved every 500 steps) \\
\bottomrule
\end{tabular}
\end{table}

\begin{table}[h]
\centering
\caption{Task-specific hyperparameters for \textbf{Dream-7B-Instruct}. $B$: batch size (prompt-completion pairs per update), $G$: completions per prompt, $K$: step budget.}
\label{tab:hyp_dream}
\small
\begin{tabular}{lccccc}
\toprule
\textbf{Task} & $B$ & $G$ & $K$ & Grad.\ Accum. & $\lambda$ \\
\midrule
Sudoku    &   192 &  16  & 8  & 2  &  1 (constant)\\
Countdown &   192 &  16 &  8 &  2 & 0.2 (halved every 500 steps) \\
GSM8K     &   192 &  16 & 8  &  2 & 0.5 (halved every 500 steps) \\
MATH500   &   192 &  16 &  8 &  2 & 0.5 (halved every 500 steps) \\
\bottomrule
\end{tabular}
\end{table}
\section{Training Dynamics}
\label{app:training_dynamics}
Figure~\ref{fig:training_curves} shows the training reward curves for
EGSPO and EGSPO-SA across four LLaDA tasks. Convergence behavior is
similar on the mathematical reasoning benchmarks (GSM8K and MATH500),
where both methods reach comparable final reward at similar rates.
On the logical reasoning benchmarks (Sudoku and Countdown), EGSPO-SA
exhibits earlier learning onset and reaches a higher asymptotic reward,
with the gap most pronounced on Sudoku. This pattern is consistent with
the main results in Section~\ref{sec:main_results}.

\begin{figure*}[t]
    \centering
    \begin{subfigure}{0.24\linewidth}
        \centering
        \includegraphics[width=\linewidth]{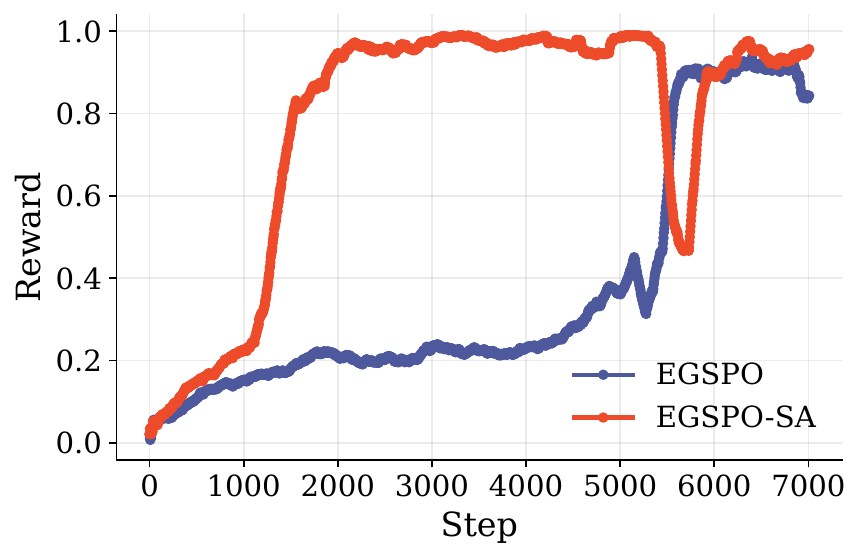}
        \caption{Sudoku}
        \label{fig:train_sudoku}
    \end{subfigure}
    \hfill
    \begin{subfigure}{0.24\linewidth}
        \centering
        \includegraphics[width=\linewidth]{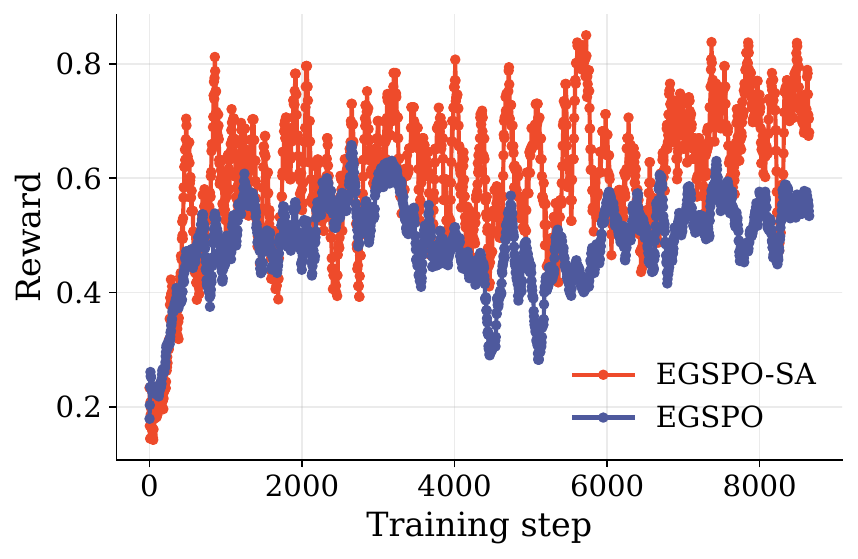}
        \caption{Countdown}
        \label{fig:train_countdown}
    \end{subfigure}
    \hfill
    \begin{subfigure}{0.24\linewidth}
        \centering
        \includegraphics[width=\linewidth]{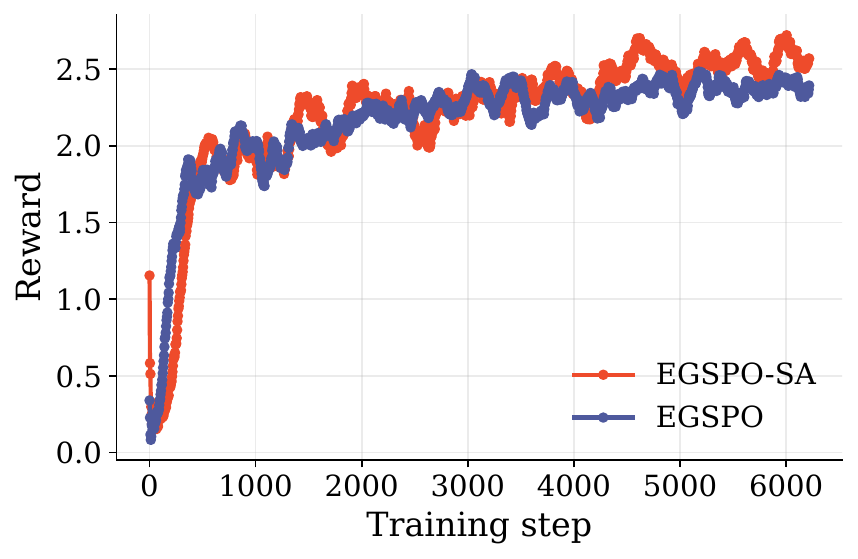}
        \caption{GSM8K}
        \label{fig:train_gsm8k}
    \end{subfigure}
    \hfill
    \begin{subfigure}{0.24\linewidth}
        \centering
        \includegraphics[width=\linewidth]{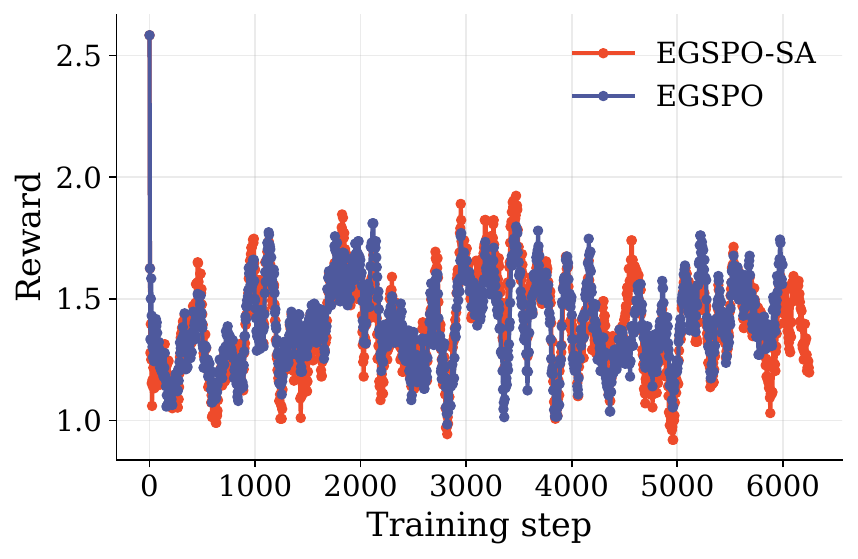}
        \caption{MATH500}
        \label{fig:train_math}
    \end{subfigure}
    \vspace{-0.2cm}
    \caption{Training curves for EGSPO and EGSPO-SA on Sudoku, Countdown, GSM8K, and MATH500 using the training settings described in Section~\ref{sec:hypsetting}.}
    \vspace{-0.3cm}
    \label{fig:training_curves}
\end{figure*}

\section{Ablation on the Advantage Weight $\lambda$}
\label{app:lambda-ablation}

The main paper uses a constant $\lambda$ within a gradient step while training. One might wonder whether a \emph{diffusion time-step} dependent weight $\lambda_t$ would help, motivated by the intuition that the one-step value approximation $\hat{V}_{t+1}^{\pi_\theta}$ is a better proxy for $V_{t+1}^{\pi_\theta}$ when $t$ is small (most tokens already unmasked) and a coarser one when $t$ is large (many tokens still masked). A natural schedule reflecting this is $\lambda_t = 1 - t/T$, which equals $0$ at the start of generation ($t = T$) and $1$ at the end ($t = 0$).

We ran an ablation on Sudoku comparing three settings: constant $\lambda = 1$, constant $\lambda = 0.5$, and the step-varying schedule $\lambda_t = 1 - t/T$. As shown in Figure~\ref{fig:lambda-ablation}, the constant $\lambda = 1$ and the step-varying schedule perform similarly throughout training and converge to comparable final accuracy. By contrast, $\lambda = 0.5$ converges more slowly and to a lower asymptote, indicating that the magnitude of $\lambda$ matters more than its step-dependence. Fixing $\lambda$ within each episode (no denoising-step dependence) is therefore the practical choice, and we did not find evidence that step-varying schedules $\lambda_t$ provide a consistent benefit. The initial value of $\lambda$ and whether to apply an across-training annealing schedule are treated as task-level hyperparameters, as detailed in Table~\ref{tab:hyp_llada}.

\begin{figure}[t]
    \centering
    \includegraphics[width=0.6\linewidth]{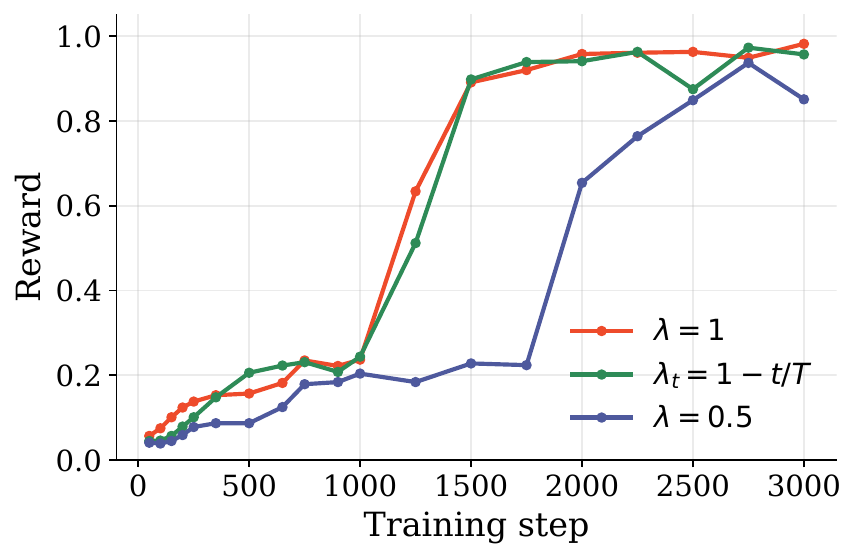}
    \caption{Ablation on the advantage weight $\lambda$ on Sudoku. Constant $\lambda = 1$ and the step-varying schedule $\lambda_t = 1 - t/T$ converge to similar accuracy; constant $\lambda = 0.5$ is noticeably slower. The magnitude of $\lambda$ matters more than its functional form across denoising steps.}
    \label{fig:lambda-ablation}
\end{figure}

\section{Algorithms}
Algorithm~\ref{alg:epsa} summarizes the proposed entropy-guided reinforcement learning procedure for diffusion language models. EGSPO-SA operates directly on the denoising process by treating each denoising step as an action and optimizing the policy using a GRPO-style objective.

The method applies entropy-based step selection to identify a subset of informative denoising timesteps for policy updates, avoiding unnecessary optimization at low-entropy, near-deterministic steps. In addition, EGSPO-SA incorporates stepwise advantages. For each selected denoising step, the partial trajectory is greedily completed to estimate the remaining reward, which is used as a baseline for credit assignment. This yields denser and more precise training signals along the denoising trajectory.

EGSPO corresponds to a special case of EGSPO-SA obtained by setting the advantage weight $\lambda = 0$, in which a single sequence-level reward is broadcast uniformly across the selected denoising steps.

\section{Qualitative Examples}
\label{app:Qualitative}

We provide qualitative examples comparing EGSPO and EGSPO-SA with d1 on representative tasks to illustrate differences in reasoning behavior and outputs.

\begin{tcolorbox}[
    colback=icmlblue10,
    colframe=icmlblue70,
    boxrule=0.4pt,
    arc=1.5pt,
    left=4pt,
    right=4pt,
    top=4pt,
    bottom=4pt,
    title={Math Example},
    fonttitle=\bfseries\footnotesize,
    fontupper=\scriptsize
]
\textbf{Problem.} Suppose that I have $6$ different books, $2$ of which are math books. In how many ways can I stack my $6$ books on a shelf if I do not want the math books to be next to each other?

\medskip
\textbf{Ground Truth.}
\[
\texttt{480}
\]
\medskip
\textbf{Model Generations.}
\begin{itemize}
    \item \textbf{EGSPO-SA (Correct)}
    \begin{quote}\ttfamily
    1. The total number of ways to arrange 6 books on a shelf is \(6! = 720\).
    
2. Treat the 2 math books as a single unit. This reduces the problem to arranging 5 units (the math book unit and the other 4 individual books).

3. The number of ways to arrange 5 units is \(5! = 120\).

4. Within the math book unit, the 2 math books can be arranged in \(2! = 2\) ways.

5. Therefore, the number of valid arrangements where the math books are next to each other is \(5! \times 2! = 120 \times 2 = 240\).

6. Subtract the number of arrangements where the math books are not next to each other from the total number of arrangements: \(720 - 240 = 480\).

So, the number of ways to stack the 6 books on a shelf such that the math books are not next to each other is \(\boxed{480}\).
    \end{quote}
    \item \textbf{EGSPO (Correct)}
    \begin{quote}\ttfamily
    First, we calculate the total number of ways to arrange the 6 books without any restrictions. This is given by the factorial of 6, which is \(6! = 720\).

Next, we calculate the number of arrangements in which the two math books are next to each other. We treat the two math books as a single unit. This reduces the problem to arranging 5 units (the block of two math books and the other 4 books). The number of ways to arrange these 5 units is \(5! = 120\).

Within the block, the two math books can be arranged in \(2! = 2\) ways. Therefore, the total number of arrangements where the two math books are next to each other is \(120 \times 2 = 240\).

To find the number of arrangements where the two math books are not next to each other, we subtract the number of arrangements where they are next to each other from the total number of arrangements:
\[
720 - 240 = 480
\]
    \end{quote}
    \item \textbf{d1 (Incorrect)}
    \begin{quote}\ttfamily
    To solve this problem, we need to consider the number of ways to arrange the 6 books such that the two math books are not next to each other. We can treat the pair of math books as a single unitblock" to simplify the problem. This means we are now arranging 5 "blocks" (the pair of math books and the other 4 books).

The number of ways to arrange these5 blocks is \(5!\) (5 factorial). However, within the pair of math books, there are 2 ways to arrange the two math books (either math first and science second, or science first and math second).

Therefore, the total number of ways to arrange the 6 books such that the math books are not next to each other is \(5! \times 2\).

Calculating this, we get:
\[
5! = 5 \times 4 \times 3 \times 2 \times 1 = 120
\]
 
 the total number of arrangements is:
\[
120 \times 2 = 240
\]
    \end{quote}
\end{itemize}
\end{tcolorbox}

\begin{tcolorbox}[
    colback=icmlblue10,
    colframe=icmlblue70,
    boxrule=0.4pt,
    arc=1.5pt,
    left=4pt,
    right=4pt,
    top=4pt,
    bottom=4pt,
    title={Sudoku Example},
    fonttitle=\bfseries\footnotesize,
    fontupper=\scriptsize
]
\textbf{Problem.} Solve the following $4\times4$ Sudoku puzzle:
\[
\texttt{0000304014022304}
\]
\medskip
\textbf{Ground Truth.}
\[
\texttt{4123324114322314}
\]
\medskip
\textbf{Model Generations.}
\begin{itemize}
    \item \textbf{EGSPO-SA (Correct)}
    \begin{quote}\ttfamily
    To solve this Sudoku puzzle, we need to fill the empty cells with the digits 1-4 such that each row, each, and each 2x2 box contains all the digits exactly once. We can start by filling the first row with the digits 3 and 4. Then, we can fill the second row with the digits 1 and 4. Next, we can fill the remaining rows and columns with the remaining digits. </reasoning> <answer> 4123324114322314 </answer>
    \end{quote}
    \item \textbf{EGSPO (Correct)}
    \begin{quote}\ttfamily
    To solve the 4x4 Sudoku puzzle, we need to fill in the empty cells with the digits 1-4 such that each row, column, and 2x2 box contains the digits 1-4 exactly once. We can start by analyzing the given puzzle and filling in the empty cells according to the rules. We need to ensure that each row, column, and 2x2 box contains the digits 1-4 exactly once. </reasoning> <answer> 4123324114322314 </answer>
    \end{quote}
    \item \textbf{d1 (Incorrect)}
    \begin{quote}\ttfamily
    To solve this 4x4 Sudoku puzzle, we need to fill in the empty cells with digits 1-4 such that each row, column, and 2x2 box contains all digits exactly once. We can start by placing the known digits and then filling in the empty cells guided by the rules of Sudoku. We will use a systematic approach to fill in the grid, ensuring that each row, column, and box conforms to the Sudoku rules. </reasoning> <answer> 1234314214322314 </answer>
    \end{quote}
\end{itemize}
\end{tcolorbox}

\begin{tcolorbox}[
    colback=icmlblue10,
    colframe=icmlblue70,
    boxrule=0.4pt,
    arc=1.5pt,
    left=4pt,
    right=4pt,
    top=4pt,
    bottom=4pt,
    title={GSM8k Example},
    fonttitle=\bfseries\footnotesize,
    fontupper=\scriptsize
]
\textbf{Problem.} Imma has 3 cats. She feeds her cats twice a day with 60 grams of cat food. How many days will 720 grams of cat food last?

\textbf{Ground Truth.}
\[
\texttt{2.0}
\]
\medskip

\medskip
\textbf{Model Generations.}
\begin{itemize}

    \item \textbf{EGSPO-SA (Correct)}
    \begin{quote}\ttfamily
    
    First, we need to determine the total amount of cat food consumed by all three cats in one day. 
    Each cat eats 60 grams of cat food per feed, and and she feeds each cat twice a day. 
    Therefore, each cat eats 60 grams * 2 feeds = 120 grams of cat food per day. 
    
    Since there are 3 cats, the total amount of cat food consumed by all cats per day is:
    120 grams perat * 3 cats = 360 grams of cat food per day.
    
    Now, we need to find out how many days 720 grams of cat food will last. 
    Therefore, we divide the total amount of cat food (720 grams) by the amount consumed in one day (360 grams).
    
    Therefore, the number of days is:
    720 grams / 360 grams/day = 2 days.
    
    Therefore, 720 grams of cat food will last for 2 days.

<answer>
\boxed{2}
</answer>

    \end{quote}

    \item \textbf{EGSPO (Correct)}
    \begin{quote}\ttfamily
    Imma has 3 cats, and she feeds each cat twice a day with 60 grams of cat food. Therefore, each cat consumes 2 * 60 grams = 120 grams of cat food per day. For 3 cats, the total food consumed per day is 3 * 120 grams = 360 grams. To find out how long a supply of 720 grams of cat food will last, we divide the total amount of food by the daily consumption: 720 grams / 360 grams = 2 days.

<answer>
\boxed{2}
</answer>

    \end{quote}

    \item \textbf{d1 (Incorrect)}
    \begin{quote}\ttfamily
    To determine how many days 720 grams of cat food will last, we need to first calculate the total amount of cat food consumed per day by allma's cats. Each cat is fed twice a day with 60 grams of food food, so, each cat consumes 60 grams per 2 meals = 30 grams per day. Since Imma has 3 cats, the total daily consumption is 30 grams per cat * 3 cats = 90 grams per day. Now, we need to find out how many days 720 grams will last by dividing the total amount of cat food by the daily consumption: 720 grams / 90 grams per day = 8 days.

<answer>
\boxed{8}
</answer>

    \end{quote}
\end{itemize}
\end{tcolorbox}

\begin{tcolorbox}[
    colback=icmlblue10,
    colframe=icmlblue70,
    boxrule=0.4pt,
    arc=1.5pt,
    left=4pt,
    right=4pt,
    top=4pt,
    bottom=4pt,
    title={Countdown Example},
    fonttitle=\bfseries\footnotesize,
    fontupper=\scriptsize
]
\textbf{Problem.} Numbers: [45, 72, 25]

Target: 52

\medskip

\medskip
\textbf{Model Generations.}
\begin{itemize}
    \item \textbf{EGSPO-SA (Correct)}
    \begin{quote}\ttfamily

    To create an arithmetic expression that evaluates to exactly 52 using the numbers 45, 72, and 25, we need to consider the operations +, -, *, and /. Let's start by checking the following:

1. The sum of 45 and 72 is 117, which is greater than 52.

2. The sum of 45 and 25 is 67, which is less than 52.

3. The sum of 72 and 25 is 97, which is greater than 52.

4. The product of 45 and 72 is greater than 52.

5. The product of 72 and 25 is 1800, which is greater than 52.

6. The difference of 72 and 45 is 27, which is less than 52.

7. The difference of 72 and 25 is 47.

<answer>
\boxed{72 - 45 + 25}
</answer>
    
    \end{quote}
    \item \textbf{EGSPO (Correct)}
    \begin{quote}\ttfamily
    
To create an arithmetic expression that evaluates to 52 using the numbers 45, 72, and 25, we need to consider the properties of the operations +, -, *, and /. Let's break down the possible combinations:

1. **Addition**: 45 + 72 = 117
2. **Subtraction**: 117 - 25 = 92
3. **Multiplication**: 72 * 25 = 1800
4. **Division**: 1800 / 45 = 40

We need to find a combination that results in 52. First, let's try subtraction and subtraction:

- 72 - 45 = 27

- 27 - 25 = 2

None of these combinations us 52. Let's try subtraction and addition:

- 72 - 45 = 27

- 27 + 25 = 52

</reasoning>

<answer>
\boxed{72 - 45 + 25}
</answer>
    
    \end{quote}
    \item \textbf{d1 (Incorrect)}
    \begin{quote}\ttfamily
    To find an arithmetic expression that evaluates to 52 using the numbers 45, 72, and 25, we need to consider the operations +, -, *, and / and ensure each number is used exactly once. Let's start by trying different combinations:

1. **, we can try adding:**

   - 45 + 72 = 117
   
   - 117 - 25 = 92
   
   - 92 - 45 = 47
   
   - 47 + 25 = 72
   
   - 72 - 45 = 27
   
   - 27 + 25 = 52

2. **, we can try subtracting:**

   - 45 - 72 = -27
   
   - -27 - 25 = -52

3. **, we can try multiplying:**

   - 45 * 72 = 3240
   
   - 3240 / 25 = 52

<answer>
\boxed{3240/25}
    
    \end{quote}
\end{itemize}
\end{tcolorbox}

\section{Limitation}
\label{sec:limitation}
While our framework provides a principled approach to RL for diffusion
language models, it has several limitations worth noting.

\textbf{Verifiable rewards only.} All experiments use tasks with
verifiable rewards (correctness on math, logical reasoning, and code),
where the reward signal is well-defined and bounded by a known constant.
Extending to settings with learned, unbounded rewards, preference-based
feedback, or non-verifiable rewards such as formal proofs would require
additional care, as the boundedness assumption underlying our entropy
bound would no longer hold tightly.

\textbf{Scope of diffusion paradigm.} Our MDP formulation and policy
gradient theorem are derived for masked diffusion language models with
monotonic unmasking, where tokens are not revised once decoded. Extending
to remasking-capable masked diffusion~\citep{wang2025remaskingdiscretediffusionmodels} or to
non-masked discrete diffusion paradigms~\citep{lou2023discrete} would
require revisiting the action space and transition dynamics of the MDP.
We leave a general formulation that unifies these settings to future work.

\section{LLM Usage}
\label{app:llm_usage}

\revision{Large language models were used to assist with editing, formatting, and revising the manuscript. The authors reviewed the resulting text and take responsibility for the accuracy of all statements, mathematical arguments, and reported results. LLMs were not used as a component of the proposed training method or experimental evaluation.}

\end{document}